\documentclass{article}

\usepackage{microtype}
\usepackage{graphicx}
\usepackage{subfigure}
\usepackage{booktabs} 

\usepackage{hyperref}

\usepackage[accepted]{icml2018}

\icmltitlerunning{Towards Faster Decentralized Learning: Less Computation and Sparse Communication}
\usepackage{amsthm}

\usepackage{amsmath}
\usepackage{amssymb}
\usepackage{bbm}
\usepackage{algorithm}
\usepackage{algorithmic}
\input{required.tex}

\def\bbA{{\ensuremath{\mathbf A}}}

\def\bbD{{\ensuremath{\mathbf D}}}

\def\bbI{{\ensuremath{\mathbf I}}}

\def\bbL{{\ensuremath{\mathbf L}}}
\def\bbM{{\ensuremath{\mathbf M}}}

\def\bbP{{\ensuremath{\mathbf P}}}
\def\bbQ{{\ensuremath{\mathbf Q}}}

\def\bbW{{\ensuremath{\mathbf W}}}
\def\bbU{{\ensuremath{\mathbf U}}}

\def\bbX{{\ensuremath{\mathbf X}}}
\def\bbY{{\ensuremath{\mathbf Y}}}
\def\bbZ{{\ensuremath{\mathbf Z}}}

\def\bbb{{\ensuremath{\mathbf b}}}

\def\bbl{{\ensuremath{\mathbf l}}}

\def\bbp{{\ensuremath{\mathbf p}}}

\def\bbx{{\ensuremath{\mathbf x}}}

\def\bbz{{\ensuremath{\mathbf z}}}
\def\bb0{{\ensuremath{\mathbf 0}}}
\usepackage{mathtools}

\usepackage{graphicx}
\usepackage{float}
\usepackage{epstopdf}
\newenvironment{shrinkeq}[1]
{ \bgroup
	\addtolength\abovedisplayshortskip{#1}
	\addtolength\abovedisplayskip{#1}
	\addtolength\belowdisplayshortskip{#1}
	\addtolength\belowdisplayskip{#1}}
{\egroup\ignorespacesafterend}
\begin{document}

\twocolumn[
\icmltitle{Towards More Efficient Stochastic Decentralized Learning: \\
           Faster Convergence and Sparse Communication}



\icmlsetsymbol{equal}{*}

\begin{icmlauthorlist}
\icmlauthor{Zebang Shen}{zju}
\icmlauthor{Aryan Mokhtari}{mit}
\icmlauthor{Tengfei Zhou}{zju}
\icmlauthor{Peilin Zhao}{scut}
\icmlauthor{Hui Qian}{zju}
\end{icmlauthorlist}

\icmlaffiliation{zju}{Zhejiang University}
\icmlaffiliation{mit}{Massachusetts Institute of Technology}
\icmlaffiliation{scut}{South China University of Technology}
\icmlcorrespondingauthor{Hui Qian}{qianhui@zju.edu.cn}

\icmlkeywords{Machine Learning, ICML}

\vskip 0.3in
]



\printAffiliationsAndNotice{}  

\begin{abstract}
	Recently, the decentralized optimization problem is attracting growing attention.
	Most existing methods are deterministic with high per-iteration cost and have a convergence rate quadratically depending on the problem condition number.
	Besides, the dense communication is necessary to ensure the convergence even if the dataset is sparse.
	In this paper, we generalize the decentralized optimization problem to a monotone operator root finding problem, and propose a stochastic algorithm named DSBA  that
	(i) converges geometrically with a rate linearly depending on the problem condition number, and (ii) can be implemented using sparse communication only.
	Additionally, DSBA handles learning problems like AUC-maximization which cannot be tackled efficiently in the decentralized setting.
	Experiments on convex minimization and AUC-maximization validate the efficiency of our method.
\end{abstract}

\section{Introduction}
Over the last decade, decentralized learning has received a lot of attention in the machine learning community due to the rise of distributed high-dimensional datasets. This paper focuses on finding a global solution to learning problems in the setting where each node merely has access to a subset of data and are allowed to exchange information with their neighboring nodes only. Specifically, consider a connected network with $N$ nodes where each node $n$ has access to a local function $f_n: \RBB^{ d}\to  \RBB$ which is the average of $q$ component functions $f_{n,i}: \RBB^{ d}\to  \RBB$, i.e. $f_n(\bbx)=(1/q)\sum_{i=1}^q f_{n,i}(\bbx)$. Considering $\bbx_n$ as the local variable of node $n$, the problem of interest is 
\begin{equation}
	\min_{\stackrel{\{\bbx_n\}_{n=1}^N}{\mathrm{s.t.} ~ \bbx_1 = \ldots = \bbx_N}}\sum_{n=1}^{N} \left\{f_n(\xB_n) \triangleq \frac{1}{q}\sum_{i=1}^{q}f_{n, i}(\xB_n)\right\}.
	\label{eqn_DCO}
\end{equation}
The formulation \eqref{eqn_DCO} captures problems in sensor network, mobile computation, and multi-agent control, where either efficiently centralizing data or globally aggregate intermediate results is unfeasible \cite{johansson2008distributed,bullo2009distributed,forero2010consensus,ribeiro2010ergodic}.

Developing efficient methods for such problem has been one of the major efforts in the machine learning community.
While early work dates back to 1980's \cite{tsitsiklis1986distributed,bertsekas1989parallel}, consensus based gradient descent and dual averaging methods with sublinear convergence have made their debut \cite{nedic2009distributed,duchi2012dual}, which consist of two steps: all nodes (i) gather the (usually dense) iterates from theirs neighbors via communication to compute a weighted average, and (ii)~update the average by the full gradient of the local $f_n$ to obtain new iterates.
Following such protocol, successors with linear convergence have been proposed recently \cite{shi2015extra,mokhtari2015decentralized,scaman2017optimal}.

Despite the progress, two entangled challenges, realized by the above interlacing steps, still remain. The first challenge is the \textit{computation complexity} of existing methods. Real world tasks commonly suffer from the \emph{ill-conditionness} of the underlying problem, which deteriorates the performance of existing methods due to their heavy dependence on the problem condition number \cite{shi2015extra,mokhtari2015decentralized}. Besides, even a single node could contain a plethora of data points, which impedes the \emph{full local gradient} evaluation required by most existing methods. The second challenge is the high \textit{communication overhead}. The existing linear convergent methods overlooked such practical issue and simply adopt the \emph{dense communication} strategy, which restrains their applications.

Furthermore, important problems like AUC maximization involves pairwise component functions which take input outside the local nodes.
Multiple rounds of communications are necessary to estimate the gradient, which precludes the direct application of existing linear convergent algorithms.
Only sublinear convergent algorithm exists \cite{colin2016gossip,ying2016stochastic}.


To bridge these gaps, we rephrase problem (\ref{eqn_DCO}) under the monotone operator framework and propose an efficient algorithm named Decentralized Stochastic Backward Aggregation (DSBA).
In the computation step of DSBA, each node computes the resolvent of a stochastically approximated monotone operator to reduce the dependence on the problem condition number.
Such resolvent admits closed form solution in problems like Ridge Regression.
In the communication step of DSBA, each node receives the nonzero components of the difference between consecutive iterates to reconstruct the neighbors' iterates.
Since the $\ell_2$-relaxed AUC maximization problem is equivalent to the minimax problem of a convex-concave function, whose differential is a monotone operator, fitting it into our formulation is seamless.
More specifically, our contributions are as follows:
\begin{enumerate}
	\item 
	DSBA accesses a single data point in each iteration and converges linearly with fast rate. The number of steps required to $\epsilon$ accurate solution is $\OM((\kappa+\kappa_g+q)\log\frac{1}{\epsilon})$, where $\kappa$ is the condition number of the problem and $\kappa_g$ is the condition number of the graph. This rate significantly improves over the existing stochastic decentralized solvers and most deterministic ones, which also holds for the $\ell_2$-relaxed AUC maximization.
	\item In contrast to the dense vector transmission in existing methods, the inter-node communication is sparse in DSBA.
	Specifically, the per-iteration communication complexity is $\OM(\rho d)$ for DSBA and $\OM(d)$ for all the other linear convergent methods, where $\rho$ is the sparsity of the dataset and $d$ is the problem dimension.
	When communication is a critical factor, our sparse communication scheme is more favorable.
\end{enumerate}
Empirical studies on convex minimization and AUC maximization problems are conducted to validate the efficiency of our algorithm. 
Improvements are observed in both computation and communication.
\section*{Notations}
We use the bold uppercase letters to denote matrices and bold lowercase letters to denote vectors.
We refer the $i^{th}$ row of matrix $\WB$ by $[\WB]_{i}$ and refer the element in the $i^{th}$ row and $j^{th}$ column by $[\WB]_{i,j}$.
$\WB^k$ denotes the $k^{th}$ power of $\WB$.
$\mathrm{Proj}_\bbU$ is the projection operator to the range of $\UB$.
\section{Related Work}
\begin{table*}[t]
	\caption{ $\kappa$ is the condition number of the problem and $\kappa_g$ be the condition number of the network graph, defined in section \ref{section_convergence_analysis}. $\Delta(\GM)$ is the max degree of the graph $\GM$. $\rho$ is the sparsity of the dataset, i.e. the ratio of nonzero elements. $\tau$ is the complexity of solving a $1$-dimensional equation, and is $\OM(1)$ in problems like Ridge Regression. All the complexity are derived for problems with linear predictor.}
	\centering
	\begin{tabular}{c|c|c|c}
		\hline
		Method &                          Convergence Rate                           &         Per-iteration Cost          & Communication Cost  \\ \hline
		EXTRA \cite{shi2015extra}  &          $\OM((\kappa^2+\kappa_g)\log\frac{1}{\epsilon})$           &     $\OM(\rho qd+\Delta(\GM)d)$     & $\OM(\Delta(\GM)d)$ \\ \hline
		 DLM \cite{ling2015dlm}   &       $\OM((\kappa^2+\kappa_g\kappa)\log\frac{1}{\epsilon})$        &     $\OM(\rho qd+\Delta(\GM)d)$     & $\OM(\Delta(\GM)d)$ \\ \hline
		 SSDA \cite{scaman2017optimal} &         $\OM(\sqrt{\kappa \kappa_g}\log\frac{1}{\epsilon})$         & $\OM(\rho qd+ q\tau +\Delta(\GM)d)$ & $\OM(\Delta(\GM)d)$ \\ \hline
		 DSA \cite{mokhtari2016dsa}  & $\OM((\kappa^4\kappa_g+\kappa_g^2+\kappa q)\log\frac{1}{\epsilon})$ &     $\OM(\rho d+\Delta(\GM)d)$      & $\OM(\Delta(\GM)d)$ \\ \hline
		 DSBA (this paper) &        $\OM((\kappa + \kappa_g + q)\log\frac{1}{\epsilon})$         &  $\OM(\rho d+ \tau +\Delta(\GM)d)$  & $\OM(\Delta(\GM)d)$ \\ \hline
		DSBA-s (this paper) &        $\OM((\kappa + \kappa_g + q)\log\frac{1}{\epsilon})$         &      $\OM(\rho d+ \tau +N^2d)$      &   $\OM(N\rho d)$    \\ \hline
	\end{tabular}
	\label{table_complexity}
\end{table*}

{\bf Deterministic Methods:} 
Directly solving the primal objective, the consensus-based Decentralized Gradient Descent (DGD) method \cite{nedic2009distributed,yuan2016convergence} has been proposed, yielding sublinear convergence rate.
EXTRA \cite{shi2015extra} improves over DGD by incorporating information from the last two iterates and is shown to converge linearly. Alternatively, D-ADMM \cite{shi2014linear} directly applies ADMM method to problem (\ref{eqn_DCO}) and achieves linear convergence. However, D-ADMM computes the proximal operator of $f_n$ in each iteration. 
To avoid such expensive proximal operator computation, \citeauthor{ling2015dlm} propose a linearized variant of D-ADMM named DLM \cite{ling2015dlm}. There also have been some efforts to exploit second-order information for accelerating convergence in ill-condition problems \cite{mokhtari2017network,eisen2017decentralized}. From the dual perspective, \cite{duchi2012dual} uses the dual averaging method and obtains a sublinear convergent algorithm. 
The work in \cite{necoara2017random} applies the random block coordinate gradient descent on the dual objective to obtain linear convergence with a rate that depends on  $\tau$, the number of blocks being selected per iteration.
When $\tau > 2$, multiple rounds of communications are needed to implement the method.
Recently, \cite{scaman2017optimal} applies the accelerated gradient descent methods on the dual problem of (\ref{eqn_DCO}) to give a method named SSDA and its multi-step communication variant MSDA and shows that the proposed methods are optimal. 
However, both SSDA and MSDA require computing the gradient of the conjugate function $f_n^*$.
All the above methods access the whole dataset in each iteration without exploiting the finite sum structure.

{\bf Stochastic Methods:} By incorporating the SAGA approximation technique, \citeauthor{mokhtari2016dsa} recently proposed a method named DSA to handle Problem (\ref{eqn_DCO}) in a stochastic manner.
In each iteration, it only computes the gradient of a single component function $f_{n, i}$, which is significantly cheaper than the full gradient evaluation \cite{mokhtari2016dsa}. DSA converges linearly, while the overall required complexity heavily depends on function and graph condition numbers.

%

We summarize the convergence rate, computation and communication cost of the aforementioned methods in Table~\ref{table_complexity}.

%
%

\section{Preliminary}


\subsection{Monotone Operator}
Monotone operator is a tool for modeling optimization problems including convex minimization \cite{rockafellar1970maximal} and minimax problem of convex-concave functions \cite{rockafellar1970monotone}.
A relation $\BM$ is a monotone operator if 
\begin{equation}
	(u - v)^\top(x -y)\geq 0, \forall (x,u), (y,v)\in \BM.
\end{equation}
$\BM$ is maximal monotone if there is no monotone operator that properly contains it.
We say an operator $\BM$ is $\mu$-strongly monotone if
\begin{equation}
\langle \BM (\xB) - \BM(\yB), \xB - \yB \rangle \geq \mu\|\xB - \yB\|^2
\end{equation}
and is $\frac{1}{L}$-cocoercive if
\begin{equation}
\langle \BM(\xB) - \BM(\yB), \xB - \yB \rangle \geq \frac{1}{L}\|\BM(\xB) - \BM(\yB)\|^2.
\end{equation}
The cocoercive property implies the maximality and the Lipschitz continuity of $\BM$, 
\begin{equation}
\|\BM(\xB) - \BM(\yB)\| \leq L\|\xB - \yB\|,
\end{equation}
but not vise versa \cite{bauschke2017convex}.
However, if $\BM$ is both Lipschitz continuous and strongly monotone, it is cocoercive.
We denote the identity operator by $\IM$ and define the resolvent $\JM_\BM$ of a maximal monotone operator $\BM$ as
\begin{equation}
\JM_\BM \triangleq (\IM + \BM)^{-1}.
\end{equation}
Finding the root of a maximal monotone operator is equivalent to find the fixed point of its resolvent:
\begin{shrinkeq}{-1ex}
\begin{equation}
	\zB^* = \JM_\BM(\zB^*) \Leftrightarrow \zB^* + \BM(\zB^*) = \zB^* \Leftrightarrow \BM(\zB^*) = \zeroB,
\end{equation}
\end{shrinkeq}{-1ex}
and when $\BM = \nabla f_{n, i}$, $\JM_\BM$ is equivalent to the proximal operator of function $f$.


\subsection{Convex-concave Formulation of AUC Maximization}
Area Under the ROC Curve (AUC) \cite{hanley1982meaning} is a widely used metric for measuring performance of classification, defined as
\begin{shrinkeq}{-1ex}
	\begin{equation}
	\sum_{i, j=1}^{q} \mathbbm{1}{\{h(\wB;\aB_i)\geq h((\wB;\aB_j)|y_i=+1,y_j=-1\}},
\end{equation}
\end{shrinkeq}{-1ex}
where $\{(\aB_j, y_j)\}_{i=j}^q$ is the set of samples, and $h(\cdot)$ is some scoring function.
However, directly maximizing AUC is NP-hard as it is equivalent to a combinatorial optimization problem \cite{gao2013one}.
Practical implementations take $h(\wB; \aB) = \aB_i^\top \wB$ and replace the discontinuous indicator function $\mathbbm{1}$ with its convex surrogates, e.g. the $\ell_2$-loss
\begin{shrinkeq}{-1ex}
	\begin{equation}
	F(\wB) =\! \frac{1}{q^+ q^-}\sum_{y_i=+1, y_j=-1}(1-\wB^\top(\aB_i-\aB_j))^2,
	\label{eqn_AUC_maximization_pairwise}
\end{equation}
\end{shrinkeq}
where $q^+$ and $q^-$ are the numbers of positive and negative instances.
However, $F(\cdot)$ comprises of pairwise losses 
\begin{equation}
	f_{i,j}(\wB) \!=\! (1-\wB^\top(\aB_i-\aB_j))^2\mathbbm{1}\{y_i=\!+1,y_j=-1\},	
\end{equation}
each of which depends on two data points.
As discussed in \cite{colin2016gossip}, minimizing (\ref{eqn_AUC_maximization_pairwise}) in a decentralized manner remains a challenging task.

For $a, b\in\RBB$, define $\bar{\wB} = [\wB, a, b] \in\RBB^{d+2}$. 
\cite{ying2016stochastic} reformulates the maximization of function (\ref{eqn_AUC_maximization_pairwise}) as
\begin{shrinkeq}{-1ex}
	\begin{equation}
	\min_{\bar{\wB}\in\RBB^{d+2}} \max_{\theta\in\RBB} F(\bar{\wB}, \theta) = \frac{1}{q}\sum_{i=1}^q f(\bar{\wB}, \theta; \aB_i),
	\label{eqn_AUC_minimax}
\end{equation}
\end{shrinkeq}
where, for $p = q^+/q$ the function $f(\bar{\wB}, \theta; \aB_i)$ is given by
\begin{align}
		 &f(\bar{\wB}, \theta; \aB_i)  = - p(1-p)\theta^2 + \frac{\lambda}{2}\|\wB\|^2 \\
		 & + (1-p)(\wB^\top\aB_i - a)^2\mathbbm{1}_{\{y_i = 1\}} + p(\wB^\top\aB_i - b)^2\mathbbm{1}_{\{y_i = -1\}}
	\nonumber\\
	&	  +2(1+\theta)(p\wB^\top\aB_i\mathbbm{1}_{\{y_i = -1\}} - (1-p)\wB^\top\aB_i\mathbbm{1}_{\{y_i = 1\}}). \nonumber
\end{align}
Such singleton formulation is amenable to decentralized framework because $f$ only depends on a single data point.

\section{Problem Formulation} \label{section_problem_formulation}
Consider a set of  $N$ nodes which create a connected graph $\GM = \{\VM, \EM\}$ with the node set $\VM = \{1, \ldots, N\}$ and the edge set $\EM=\{(i,j) \mid \text{if} \ i,j \ \text{are connected}\}$. We assume that the edges are reciprocal, i.e., $(i,j)\in \EM$ iff $(j,i)\in \EM$ and denote $\NM_n$ as the neighborhood of node $n$, i.e. $\NM_n = \{m: (m,n) \in \EM\}$.

For the decision variable $\zB \in \RBB^d$, consider the problem of finding the root of the operator $\sum_{n=1}^{N} \BM_n(\bbz) $, where the operator $\BM_n: \RBB^{ d}\mapsto \RBB^{ d}$ is only available at node $n$ and is defined as the sum of $q$ Lipschitz continuous strongly monotone operators $\BM_{n,i}: \RBB^{ d}\mapsto \RBB^{ d}$.

To handle this problem in a decentralized fashion we define $\bbz_n$ as the local copy of $\bbz$ at node $n$ and solve the program
\begin{equation}\label{eqn_decentralized_consensus_root_monotone_operator}
	\find_{\stackrel{\{\bbz_n\}_{n=1}^N}{\bbz_1 = \ldots = \bbz_N}} \sum_{n=1}^{N} \BM_n(\bbz_n) = \sum_{n=1}^{N} \frac{1}{q}\sum_{i=1}^{q} \BM_{n, i}(\bbz_n) = \zeroB.
\end{equation}
The finite sum minimization problem \eqref{eqn_DCO} is a special case of \eqref{eqn_decentralized_consensus_root_monotone_operator} by setting $\BM_{n,i}= \nabla f_{n, i}$, and the $\ell_2$-relaxed AUC maximization (\ref{eqn_AUC_minimax}) is captured by choosing $\BM_{n, i}(\zB) = [\frac{\partial f}{\partial \bar{\wB}}; -\frac{\partial f}{\partial \theta}]$ with $\zB = [\bar{\wB}; \theta]$.
Since $\BM_{n, i}$ is strongly monotone and Lipschitz continuous, it is cocoercive \cite{bauschke2017convex}.

To have a more concrete understanding of the problem, we first introduce an equivalent formulation of Problem~\eqref{eqn_decentralized_consensus_root_monotone_operator}. 
Define the matrix $\bbZ=[\bbz_1^\top;\dots;\bbz_n^\top] \in \RBB^{N\times d}$ as the concatenation of the local iterates $\bbz_n$ and the operator $\BM(\bbZ): \RBB^{N\times d}\mapsto \RBB^{N\times d}$ as $\BM(\bbZ):=[\BM_1(\bbz_1)^\top;\dots;\BM_n(\bbz_n)^\top]$.
Consider the mixing matrix $\bbW=[w_{m,l}]\in \RBB^{N\times N}$ satisfying the following conditions, which satisfies

\quad {\bf (i) (Graph sparsity)} If $m \notin \NM_l$, then $w_{m,l} = 0$;

\vspace{-1mm}
\quad {\bf (ii) (Symmetry)} $\bbW = \bbW^\top$;

\vspace{-1mm}
\quad{\bf (iii) (Null space property)} $\mathrm{null} (\bbI - \bbW) = \mathrm{span} \{\oneB_N\}$;

\vspace{-1mm}
\quad{\bf (iv) (Spectral property)} $0 \preccurlyeq \bbW \preccurlyeq \bbI_N$.

It can be shown that Problem~\eqref{eqn_decentralized_consensus_root_monotone_operator} is equivalent to
\begin{equation}\label{eqn_decentralized_consensus_root_monotone_operator_version_2}
\begin{aligned}
\find_{\bbZ\in \RBB^{N\times d}}& \ \ \BM(\bbZ)^\top\oneB_N = \zeroB_d\\
\text{subject to}&\ \ (\bbI_N-\bbW)\bbZ=\bb0_{N\times d}.
\end{aligned}
\end{equation}
This is true since $\Null(\IB-\WB)=\Span(\oneB_N)$ and therefore the condition $(\bbI-\bbW)\bbZ=\bb0$ implies that a matrix $\bbZ$ is feasible iff $ \bbz_1 = \ldots = \bbz_N$.

If we define $\bbU\triangleq(\bbI-\bbW)^{1/2}$, the optimality conditions of Problem~\eqref{eqn_decentralized_consensus_root_monotone_operator_version_2} imply that there exists some $\PB^* \in \RBB^{N\times d}$, such that for $\QB^* = \bbU \PB^*$ and $\alpha>0$	
\begin{equation}
		\bbU \QB^* + \alpha \BM(\ZB^*) = \zeroB ~\mathrm{and}~ -\bbU \ZB^* = \zeroB,
		\label{eqn_optimality}
	\end{equation}
where $\bbZ^*\in \RBB^{N\times d}$ is a solution of Problem~\eqref{eqn_decentralized_consensus_root_monotone_operator_version_2}.
Note that $\mathrm{span}(\bbI - \bbW) = \mathrm{span}(\bbU)$.
The first equation of (\ref{eqn_optimality}) depicts the optimality of $\bbZ^*$: if $\bbZ^*$ is a solution, every column of $\BM(\bbZ^*)$ is in $\mathrm{span}\{\oneB_N\}^\perp = \mathrm{span}(\bbU)$ and hence there exists $\bbP\in\RBB^{N \times d}$ such that $\UB\bbP + \alpha \BM(\ZB^*) = \zeroB$.
We can simply take $\bbQ^* = \mathrm{Proj}_\bbU\bbP$ which gives $\bbU\bbQ^* = \bbU\bbP$.
The second equation of (\ref{eqn_optimality}) describes the consensus property of $\bbZ^*$ and is equivalent to the constraint of Problem (\ref{eqn_decentralized_consensus_root_monotone_operator_version_2}).

Using \eqref{eqn_optimality}, we formulate Problem (\ref{eqn_decentralized_consensus_root_monotone_operator}) as finding the root of the following operator
\begin{equation}
\TM(\bbA) = \bigg(
\underbrace{\begin{bmatrix}
	\BM& 0 \\
	0& 0
	\end{bmatrix}}_{\TM_1} +
\underbrace{\frac{1}{\alpha}\begin{bmatrix}
	\zeroB & \bbU \\
	-\bbU& \zeroB
	\end{bmatrix}}_{\TM_2}
\bigg)
\underbrace{\begin{bmatrix}
	\ZB \\
	\YB
	\end{bmatrix}}_{\bbA},
\label{eqn_augmented_operator}
\end{equation}
where the augmented variable matrix $\bbA \in \RBB^{2N \times d}$ is obtain by concatenating $\bbZ$ with $\bbY \in \RBB^{N \times d}$. Using the result in \cite{davis2015convergence}, it can be shown that $\TM$ is a maximally monotone operator, and hence its resolvent $\JM_\TM$ is well defined.
Unfortunately, directly implementing the fixed point iteration 
$	\bbA^{t+1} = \JM_\TM (\bbA^t)$ 
requires access to global information which is infeasible in decentralized settings. Inspired by \cite{wu2016decentralized}, we introduce the positive definite matrix 
\begin{equation}
\bbD \triangleq \frac{1}{\alpha} \begin{bmatrix}
\bbI & \bbU \\
\bbU & \bbI
\end{bmatrix},
\end{equation}
and use the fixed point iteration of the resolvent of $\bbD^{-1}\TM$ to find the root of (\ref{eqn_augmented_operator}) according to the recursion
\begin{equation}
	\bbA^{t+1} = \JM_{\bbD^{-1}\TM} (\bbA^t)
	\label{eqn_resolvent_fix_point_iteration}.
\end{equation}
Note that since $\bbD$ is positive definite, $\bbD^{-1}\TM$ shares the same roots with $\TM$, therefore the solutions of the fixed point updates of $\JM_{\bbD^{-1}\TM} $ and $\JM_{\TM}$ are identical. 

The main advantage of the recursion in \eqref{eqn_resolvent_fix_point_iteration} is that it can be implemented with a single round of local communication only.
However, (\ref{eqn_resolvent_fix_point_iteration}) is usually computationally expensive to evaluate.
For instance, when $\BM_{n, i} = \nabla f_{n, i}$, (\ref{eqn_resolvent_fix_point_iteration}) degenerates to the update of P-EXTRA \cite{shi2015proximal}, which computes the proximal operator of $f_n = \frac{1}{q}\sum_{i=1}^{q} f_{n, i}$ in each iteration.
The evaluation of such proximal operator is considered computational costly in general, especially for large-scale optimization. 

In the following section, we introduce an alternative approach that improves the update in (\ref{eqn_resolvent_fix_point_iteration}) in terms of both computation and communication cost by stochastically approximating $\TM$.

\section{Decentralized Stochastic Backward Aggregation}

In this section, we propose the Decentralized Stochastic Backward Aggregation (DSBA) algorithm for Problem~\eqref{eqn_decentralized_consensus_root_monotone_operator}.
By exploiting the finite sum structure of each $\BM_n$ and the sparsity pattern in component operator $\BM_{n, i}$, DSBA yields lower per-iteration computation and communication cost.

Let $i_n^t$ be a random sample, approximate $\BM_n(\zB)$ by
\begin{equation}
\hat{\BM}^t_n(\zB) = \BM_{n, i_n^t}(\zB) - \phi^t_{n, i_n^t} + \frac{1}{q}\sum_{i=1}^{q}\phi^t_{n, i},
\end{equation}
where $\phi^t_{n, i} = \BM_{n, i}(\yB_{n, i}^t)$ is the history operator output maintained in the same manner with SAGA \cite{defazio2014saga}.
We further denote $\bar{\phi}_n^t = \frac{1}{q}\sum_{i=1}^{q}\phi^t_{n, i}$.
Using such definition $\hat{\BM}^t_n(\zB)$, we replace the operators ${\BM}$ and ${\TM}_1$ in (\ref{eqn_augmented_operator}) by their approximate versions, defined as
\begin{equation}
\hat{\BM}^t(\ZB) = \begin{bmatrix}
\hat{\BM}^t_1(\zB_1) \\
\vdots \\
\hat{\BM}^t_N(\zB_N)
\end{bmatrix}
~\mathrm{and}~
\hat{\TM}_1^t = \begin{bmatrix}
\hat{\BM}^t& 0 \\
0& 0
\end{bmatrix}.
\end{equation}
Hence, the fixed point update (\ref{eqn_resolvent_fix_point_iteration}) is changed to
\begin{equation}
\bbA^{t+1} = (\IM + \bbD^{-1}(\hat{\TM}_1^t + \TM_2))^{-1} (\bbA^t),
\end{equation}
which by plugging in the definitions of $\bbA$, $\bbD$, $\hat{\TM}_1^t$, and $ \TM_2$ can be written as
\begin{align}
\ZB^{t+1} + 2\bbU\YB^{t+1} &= \ZB^t - \alpha\hat{\BM}^t(\ZB^{t+1}) + \bbU\YB^t, \label{eqn_derivation_I}\\
\YB^{t+1} &= \bbU\ZB^t + \YB^t. \label{eqn_derivation_II}
\end{align}
Computing the difference between two consecutive iterations of (\ref{eqn_derivation_I}) and using (\ref{eqn_derivation_II})
lead to the update of the proposed DSBA algorithm, for $t>1$,
\begin{shrinkeq}{-1ex}
	\begin{equation}\label{eqn_update_derivation_t}
\ZB^{t+1}\! \triangleq  2\tilde{\bbW}\ZB^t -\tilde{\bbW}\ZB^{t-1}\!\!- \alpha(\hat{\BM}^t(\ZB^{t+1}) - \hat{\BM}^{t-1}(\ZB^{t})),
\end{equation}
\end{shrinkeq}
where $\tilde{\WB} = (\bbW + \bbI)/2 = [\tilde{w}_{m, n}]\in\RBB^{N\times N}$. By setting $\YB^0 = \zeroB$, the update for step $t = 0$ is given by 
\begin{shrinkeq}{-1ex}
	\begin{equation}
\ZB^1 \triangleq \bbW\ZB^0 - \alpha\hat{\BM}^0(\ZB^{1}). 
\label{eqn_update_derivation_0}
\end{equation}
\end{shrinkeq}
\begin{algorithm}[t]
\caption{DSBA for node $n$}
\begin{algorithmic}[1]
	\REQUIRE consensus initializer $\zB^0$, step size $\alpha$, $\bbW$, $\tilde{\bbW}$;
	\STATE For all $i \in [q]$, initialize $\phi_{n, i}^0 = \BM_{n, i}(\zB^0)$, set $\delta^0_n = 0$;
	\FOR{$t = 0,1,2, \dots$}
	\STATE Gather the iterates $\zB_m^t$ from neighbors $m \in \NM_n$;
	\STATE Choose $i_n^t$ uniformly at random from the set $[q]$;
	\STATE Update $\psi_n^t$ according to (\ref{eqn_update_derivation_0_single_simple_phi}) ($t=0$) or (\ref{eqn_psi_def}) ($t>0$)
	\STATE Compute $\zB_n^{t+1}$ from (\ref{eqn_update_derivation_t_single_simple_z});
	\STATE Compute $\delta_n^t = \BM_{n, i_n^t}(\zB_n^{t+1}) - \phi^t_{n, i_n^t}$;
	\STATE Set $\phi^{t+1}_{i_n^t} = \BM_{n, i_n^t}(\zB_n^{t+1})$ and $\phi^{t+1}_{i} = \phi^t_{i}$ for $i\neq i_n^t$;
	\ENDFOR
\end{algorithmic}
\label{alg_main}
\end{algorithm}
\noindent{\bf Implementation on Node $n$}\\
We now focus on the detailed implementation on a single node $n$. 
The local version of the update \eqref{eqn_update_derivation_t} writes
\begin{shrinkeq}{-1ex}
	\begin{equation}
	\zB_n^{t+1}\!\triangleq \!\!\sum_{m \in \NM_n}\!\!\tilde{w}_{n, m}(2\zB_m^t - \zB_m^{t-1}) - \alpha(\hat{\BM}_n^t(\zB_n^{t+1}) - \hat{\BM}_n^{t-1}(\zB_n^{t})).
	\label{eqn_update_derivation_t_single}
	\end{equation}
\end{shrinkeq}
This update can be further simplified. 
Using the definition
\begin{equation}
\delta_n^t \triangleq \BM_{n, i_n^t}(\zB_n^{t+1}) - \phi^t_{n, i_n^t},
\label{eqn_def_delta}
\end{equation} 
we have $\hat{\BM}_n^{t+1} - \hat{\BM}_n^{t} = \delta_n^t- \frac{q-1}{q}\delta_n^{t-1}$, and therefore the update in \eqref{eqn_update_derivation_t_single} can be simplified to
\begin{shrinkeq}{-1ex}
	\begin{equation}
\zB_n^{t+1}\! \triangleq\! \sum_{m \in \NM_n}\!\tilde{w}_{n, m}(2\zB_m^t - \zB_m^{t-1}) + \alpha( \frac{q-1}{q}\delta_n^{t-1}- \delta_n^t).
\label{eqn_update_derivation_t_single_simple}
\end{equation}
\end{shrinkeq}
Note that $\delta_n^t$ shares the same nonzero pattern as the dataset and is usually sparse.
For the initial step $t = 0$, since $\hat{\BM}_n^1 = \delta_n^0 + \bar{\phi}_n^0$, we have 
$\zB_n^{1} = \sum_{m \in \NM_n}w_{n, m}\zB_m^0 - \alpha(\delta_n^0 + \bar{\phi}_n^0)$. 
However, we cannot directly carry out \eqref{eqn_update_derivation_t_single_simple} since $\delta_n^t$ involves the unknown $\zB_n^{t+1}$. To resolve this issue we define for $t\geq1$
\begin{shrinkeq}{-1ex}
	\begin{equation}\label{eqn_psi_def}
		\psi_n^t\triangleq\!\sum_{m \in \NM_n} \! \tilde{w}_{n, m}(2\zB_m^t - \zB_m^{t-1}) + \alpha( \frac{q-1}{q}\delta_n^{t-1}\! +\phi^t_{n, i_n^t}).
	\end{equation}
\end{shrinkeq}
Using \eqref{eqn_psi_def} and \eqref{eqn_update_derivation_t_single_simple}, it can be easily verified that $ \zB_n^{t+1}+\alpha\BM_{n, i_n^t}(\zB_n^{t+1})={\psi_n^t}$, therefore $\zB_n^{t+1}$ can be computed as
\begin{equation}
\zB_n^{t+1} = \JM_{\alpha\BM_{n, i}}(\psi_n^t)= (I + \alpha \BM_{n, i_n^t})^{-1}(\psi_n^t). \label{eqn_update_derivation_t_single_simple_z}
\end{equation}
Indeed, the outcome of the updates in \eqref{eqn_psi_def}-\eqref{eqn_update_derivation_t_single_simple_z} is equivalent to the update in \eqref{eqn_update_derivation_t_single_simple}, and they can be computed in a decentralized manner. Also, for the initial step $t=0$, the variable $\bbz_n^1$ can be computed according to \eqref{eqn_update_derivation_t_single_simple_z} with 
\begin{equation}
{\psi_n^0}:=\sum_{m \in \NM_n}w_{n, m}\zB_m^0 + \alpha(\phi^0_{n, i_n^0} - \bar{\phi}_n^0).
\label{eqn_update_derivation_0_single_simple_phi}
\end{equation}
The resolvent (\ref{eqn_update_derivation_t_single_simple_z}) can be obtained by solving a one dimensional equation for learning problems like Logistic Regression, and admits closed form solution for problems like least square and $\ell_2$-relaxed AUC maximization.
DSBA is summarized in Algorithm~\ref{alg_main}.

\begin{remark}
	DSBA is related to DSA in the case that $\BM_{n, i}$ is the gradient of a function, i.e. $\BM_{n, i} = \nabla f_{n, i}$.
	In each iteration, if we compute $\delta_n^t$ with 
	\begin{equation}
	\delta_n^t = \BM_{n, i_n^t}(\zB_n^t) - \phi^t_{n, i_n^t},
	\end{equation}
	i.e. we evaluate $\BM_{n, i_n^t}$ at the $\zB_n^t$ instead of $\zB_n^{t+1}$, we recover the DSA method \cite{mokhtari2016dsa}.
	In such gradient operator setting, when the is only a single node, DSBA degenerates to the Point-SAGA method \cite{defazio2016simple}.
\end{remark}
\vspace{-3ex}
\subsection{Implementation with Sparse Communication}
In existing decentralized methods, nodes need to compute the weighted averages of their neighbors' iterates, which are dense in general.
Therefore a $d$-dimensional full vector must be transmitted via every edge $(m, l)\in\EM$ in each iteration.

In this section, we assume the output of every component operator $\BM_{n, i}$ is $\rho$-sparse, i.e. $nnz(\BM_{n, i}(\zB))/d \leq \rho$ for all $\zB\in\RBB^d$ and show that DSBA can be implemented by only transmitting the usually sparse vector $\delta_n^t$ (\ref{eqn_def_delta}).

WLOG, we take the perspective of node $0$ to describe the communication and computation strategies.
First, we define the topological distance $\xi_i$ from node $i$ to node $0$ by
\begin{shrinkeq}{-1ex}
	\begin{equation}
\argmin_{k\in\NBB} ~\mathrm{s.t.}~ [W^k]_{0, i}\neq 0,
\end{equation}
\end{shrinkeq}
and we have, $[W^k]_{0, i}\! =\! 0$ for all node $i$ with distance $\xi_i\!>\!k$.
Let the diameter of the network be $E = \max_{i\in\VM} \xi_i$.

All the communication in the network happens when computing $\psi_0^t$. For $n=0$ and we unfold the iteration (\ref{eqn_psi_def}) by the definition of $\ZB^t$ in (\ref{eqn_update_derivation_t}),
\begin{shrinkeq}{-2ex}
	\begin{align}\label{eqn_psi_unfold}
\psi_0^t  &= 2^E[W^E]_0\ZB^{t-E}- \sum_{\tau = 1}^{E}2^{\tau-1}[W^\tau]_0\ZB^{t-\tau-1}\nonumber\\
& + \sum_{\tau = 1}^{E}2^{\tau}[W^\tau]_0\DeltaB^{t-\tau}\! + \alpha(\frac{1-q}{q}\delta_0^{t-1}+\phi^t_{0, i_0^t}), \nonumber\\
& = \textcircled{1} + \textcircled{2} + \textcircled{3} + \textcircled{4}
\end{align}
\end{shrinkeq}
where 
$
\DeltaB^{t} = [(\delta_0^t - \delta_0^{t-1})^\top; \ldots; (\delta_{N-1}^t - \delta_{N-1}^{t-1})^\top].
$
Suppose that we have a communication strategy that satisfies the following assumption: before evaluating (\ref{eqn_psi_unfold}), node $0$ has the set  $\QM_t = \{\delta^\tau_{n}: \tau+\xi_n \leq t, n\neq 0\}$.
\textcircled{3} can be computed because computing each term of \textcircled{3}, $[W^\tau]_0\DeltaB^{t-\tau}$, only needs \{$\delta_{i}^{t-\tau}: \xi_i \leq \tau$\}.
Further, if we can inductively ensure that \textcircled{1} and every term \textcircled{2} are in the memory of node $0$ before computing (\ref{eqn_psi_unfold}),  $\psi_0^t$ can be computed since \textcircled{4} is local information.

In the following, we introduce a communication strategy that satisfies the assumption and show that the inductions on \textcircled{1} and \textcircled{2} holds.\\
\noindent{\bf Communication:} We group the nodes based on the distance: $\VM_j = \{n \in \{0, \ldots, N-1\}: \xi_n = j\}$.
Define the set $\GM^t_j \triangleq \{\delta_n^t: n \in \VM_j\}$.
Let $\FM_E^t \triangleq \{\GM_E^t\}$, we recursively define $\FM_j^t \triangleq \FM_{j+1}^{t-1} \cup \{\GM^t_j\}$.
Our communication strategy is, in the $t^{th}$ iteration, $\VM_j$ sends the set $\FM_j^t = \FM_{j+1}^{t-1} \cup \{\GM^t_j\} = \{\GM^\tau_i: i+\tau = t+j, i\geq j\}$ to $\VM_{j-1}$.
From such strategy, in iteration $t$, node $0$ receives from $\VM_{j}$ the set $\FM_{1}^t = \{\GM^\tau_i: i+\tau = t, i\geq1\}$.
Note that if $\delta_n^\tau$ appears in multiple neighbors of node $0$, only the one with the minimum node index sends it to node $0$.
Since $\QM_t = \cup_{\tau \leq t}\FM_{1}^\tau$, the desired set is obtained.

\noindent{\bf Computation:}
We now inductively show that \textcircled{1} and \textcircled{2} can be computed.
At the beginning of iteration $t$, assume that $\{[\tilde{\WB}^\tau]_0\ZB^{t-\tau}, \tau \in [E]\}$, $\ZB^{t-E - 1}$, and $\ZB^{t-E - 2}$ are maintained in memory. 
According to the above discussion, $\psi_0^t$ can be computed and hence $\zB_0^{t+1}$ and $\delta_0^t$ can be obtained.
To maintain the induction, we compute $\ZB^{t-E}$ by its definition (\ref{eqn_update_derivation_t_single_simple}) where $\{\delta_n^{t-E}, n = 0, \ldots, N-1\}$ (by the communication strategy), $\ZB^{t-E - 1}$, and $\ZB^{t-E - 2}$ (by induction) are already available to node $0$.
To obtain $\{[W^\tau]_0\ZB^{t-\tau}, \tau \in [E-1]\}$, we compute $[W^E]_0\ZB^{t-E}$ first and then compute recursively
\begin{align}\label{eqn_inductive_computation}
[W^{\tau-1}]_0\ZB^{t-\tau+1} &= 2[W^\tau]_0\ZB^{t-\tau} - [W^\tau]_0\ZB^{t-\tau -1} 
\nonumber\\
&\qquad+ [W^{\tau-1}]_0\DeltaB^{t - \tau + 1},
\end{align}
for $\tau = E, \ldots, 2$, where the first term is from induction,  the second term is in memory, and the last term is computed in \textcircled{3}.
We summarize our strategy in Algorithm \ref{alg_computation}.

As the choice of node $0$ is arbitrary, we use the aforementioned communication and computation strategies for all nodes.
By induction, if each node $n$ generates $\delta_n^{t-1}$ correctly at iteration $t-1$, we can show that $\psi_n^t$ and hence $\delta_n^{t}$ can also be correctly computed in the same manner.
The computation complexity at each node is $\OM(dN^2)$, dominated by step 1 in Algorithm \ref{alg_computation}.

The average communication complexity is of $\OM(Nd\rho)$. 
WLOG, use node $0$ as a proxy of all nodes. The computation part requires the set $\FM_1^t$ to be received by node $0$ at time $t$. 
Removing the duplicate, we have $|\FM_1^t|\leq N$. 
Hence, the number of DOUBLEs received by node $0$ is of $\OM(Nd\rho)$. Further, since that of data sent by all nodes equals to the amount of data received by all nodes, we have the result.

The local storage requirement of DSBA is $O(qd\rho+Nd)$. 
Aside from the $O(\rho qd)$ storage for the dataset, a node stores a delayed copy of other nodes which costs a memory of $\OM(Nd)$, and due to the use of linear predictor the cost of storing gradient information at each node is $\OM(q)$, \cite{schmidt2017minimizing}. Hence, the overall required storage is $\OM(qd\rho+Nd+q)$. As $\rho\cdot d$ is the number of nonzero elements in the vector it follows that $\rho\cdot d>=1$ and hence $\OM(q)\leq\OM(qd\rho)$. Further, if we assume every sample has more than N nonzero entries, $\OM(qd\rho)$ dominates $\OM(Nd)$ as well, we need a memory of $\OM(qd\rho)$.
\renewcommand{\algorithmicrequire}{\textbf{Require:}}
\renewcommand{\algorithmicensure}{\textbf{Ensure:}}
\begin{algorithm}[t]
	\begin{algorithmic}[1]
		\REQUIRE $\{[\tilde{\WB}^\tau]_0\ZB^{t-\tau}, \tau \in [E]\}, \ZB^{t-E - 1}, \ZB^{t-E - 2}$
		\STATE Compute $\ZB^{t-E}$ from its definition;
		\STATE Compute $\psi_0^t$ from (\ref{eqn_psi_unfold}) and $\zB_0^{t+1}$ from (\ref{eqn_update_derivation_t_single_simple_z});
		\STATE $\delta_0^t = B_{0, i_0^t}(\zB_0^{t+1}) - \phi^t_{0, i_0^t}$, update the gradient table;
		\STATE Compute $\{[\tilde{\WB}^\tau]_0\ZB^{t-\tau+1}, \tau \in [E]\}$ from (\ref{eqn_inductive_computation});
		\ENSURE $\ZB^{t-E}, \zB_0^{t+1}, \delta_0^t, \{[\tilde{\WB}^\tau]_0\ZB^{t-\tau+1}, \tau \in [E]\}$
	\end{algorithmic}
	\caption{Computation on node $0$ at iteration $t$}
	\label{alg_computation}
\end{algorithm}
\section{Convergence Analysis} \label{section_convergence_analysis}
In this section, we study the convergence properties of the proposed DSBA method. To achieve this goal, we define a proper Lyapunov function for DSBA and prove its linear convergence to zero which leads to linear convergence of the iterates $\bbz_n^t$ to the optimal solution $\bbz^*$. To do so, first we define $\bbM $ and the sequence of matrices $\bbQ^t$ and $\bbX^t$ as
	\begin{equation}\label{eqn_def_Q_X}
\MB \triangleq \begin{bmatrix}
\tilde{\WB} &  \ZEROBB \\
\ZEROBB & I
\end{bmatrix}
, \ \ 
\QB^t \triangleq \sum_{k=0}^{t} \bbU\ZB^k
, \ \ 
\XB^t \triangleq \begin{bmatrix}
\ZB^t \\
\bbU\QB^t
\end{bmatrix}.
\end{equation}
Recall the definition of $\bbQ^*$ in \eqref{eqn_optimality} and define $\bbX^*$ as the concatenation of $\ZB^*$ and $\UB\QB^*$, i.e., 
$
\XB^* = [
\ZB^* ;
\UB\QB^*
]
$. 

\begin{lemma}\label{lemma_iteration_recursion}
	Consider the proposed DSBA method defined in Algorithm \ref{alg_main}. By incorporating the definitions of the matrices $\QB^t$ and $\XB^t$ in \eqref{eqn_def_Q_X}, it can be shown that
		\begin{equation}
	\alpha [\hat{\BM}^{t}(\ZB^{t+1}) \!-\! \BM(\ZB^*)]\! =\! \tilde{\WB}(\ZB^t \!-\!\ZB^{t+1})\! -\! \UB(\QB^{t+1} \!\!-\! \QB^*),
	\end{equation}
	and 
	\begin{shrinkeq}{-1ex}
		\begin{align}\label{eqn_sth}
	&2\langle \ZB^{t+1} - \ZB^*, \alpha [\BM(\ZB^*) - \hat{\BM}^{t}(\ZB^{t+1})]\rangle \\
	&= \|\XB^{t+1} - \XB^*\|_\bbM^2 + \|\XB^{t+1} - \XB^t\|_\bbM^2 - \|\XB^t - \XB^*\|_\bbM^2.\nonumber
	\end{align}
	\end{shrinkeq}
\end{lemma}
\begin{proof}
See Section \ref{app:lemma_iteration_recursion} in the supplementary material.
	\end{proof}

The result in Lemma \ref{lemma_iteration_recursion} shows the relation between the norm $\|\XB^{t+1} - \XB^*\|_\bbM^2$ and its previous iterate $\|\XB^{t} - \XB^*\|_\bbM^2$. Therefore, to analyze the speed of convergence for $\|\XB^{t} - \XB^*\|_\bbM^2$ we first need to derive bounds for the remaining terms in \eqref{eqn_sth}. To do so, we need to define a few more terms.
By selecting component operator $i$ on node $n$ in the $t^{th}$ iteration, we define $\zB_{n, i}^{t+1}$ for all $i\in [q]$ as
\begin{shrinkeq}{-1ex}
	\begin{equation*}
\zB_{n, i}^{t+1} \triangleq\!\! \sum_{m \in \NM_n} \! \tilde{w}_{n, m}(2\zB_m^t - \zB_m^{t-1}) - \alpha(\hat{\BM}_n^{t}(\zB_{n, i}^{t+1}) - \hat{\BM}^{t-1}(\zB_n^t)).
\end{equation*}
\end{shrinkeq}
Computing $\zB_{n, i}^{t+1}$ requires to evaluate the resolvent of $\BM_{n, i}$, but here we only define it for the analysis.
In the actual procedure, we only select $i = i_n^t$, compute $\zB_{n, i_n^t}^{t+1}$ , and set $\zB_n^{t+1} = \zB_{n, i_n^t}^{t+1}$.
Having such definition, we define two nonnegative sequences that are crucial to our analysis:
\vspace{-1ex}
	\begin{equation}\label{eqn_def_S}
	S^{t} \triangleq \sum_{n=1}^{N}\frac{2}{q}\sum_{i=1}^{q}\|\BM_{n, i}(\zB^{t}_{n, i}) - \BM_{n, i}(\zB^*)\|^2,
	\end{equation}
	\vspace{-4mm}
	\begin{equation}\label{eqn_def_T}
T^{t} \triangleq \sum_{n=1}^{N} \frac{2}{q}\sum_{i=1}^{q}\langle \zB_{n, i}^{t} - \zB^*, \BM_{n, i}(\zB_{n, i}^{t}) - \BM_{n, i}(\zB^*) \rangle,
\vspace{-1ex}
\end{equation}
where the nonnegativity of the sequence $T^{t}$ is due to the monotonicity of each component operator $\BM_{n, i}$. Define $D^t$ as the component-wise discrepancy between the historically evaluated stochastic gradients and gradients at the optimum
	\begin{equation}\label{eqn_def_D}
D^t = \sum_{n=1}^N\frac{2}{q}\sum_{i=1}^{q}\|\BM_{n, i}(\yB^t_{n, i}) - \BM_{n, i}(\zB^*)\|^2,
\end{equation}
where $\BM_{n, i}(\yB^t_{n, i})$ is maintained by the SAGA strategy.
In the following lemma, we derive an upper bound on the expected inner product $\EBB\langle \ZB^{t+1} - \ZB^*, \BM(\ZB^*) - \hat{\BM}^{t}(\ZB^{t+1})\rangle $.
\begin{lemma}\label{lemma_iteration_recursion_upper_bound}
	Consider the proposed DSBA method defined in Algorithm \ref{alg_main}. Further, recall the defintions of the sequences $S^{t}$, $T^{t}$, and $D^t$ in \eqref{eqn_def_S}, \eqref{eqn_def_T}, and \eqref{eqn_def_D}, respectively. If each operator $\BM_{n, i}$ is $(1/L)$-cocoercive, it holds 	for any $0\leq \theta \leq 1$ and $\eta>0$ that 
	\begin{align}
	&\EBB\langle \ZB^{t+1} - \ZB^*, \BM(\ZB^*) - \hat{\BM}^{t}(\ZB^{t+1})\rangle \\
	\leq &\frac{1}{2\eta}\EBB\|\ZB^{t+1} - \ZB^t\|^2 + \frac{\eta}{4}D^t -\frac{\theta}{2L}S^{t+1} - \frac{1-\theta}{2} T^{t+1}.\nonumber
	\end{align}
\end{lemma}
\vspace{-3ex}
\begin{proof}
See Section \ref{app:lemma_iteration_recursion_upper_bound} in the supplementary material.
\end{proof}
\vspace{-2ex}
The next lemma bounds the discrepancy between the average of the historically evaluated stochastic gradients and the gradients at the optimal point.
\begin{lemma}\label{lemma_G_tp1}
	Consider the DSBA method outlined in Algorithm \ref{alg_main}. From the construction of $\hat{\BM}^{t}$ and the definitions of $S^{t}$ and $D^{t}$ in \eqref{eqn_def_S} and \eqref{eqn_def_D}, respectively, we have for $t\geq 0$,
		\begin{equation}
	\EBB\|\hat{\BM}^{t}(\ZB^{t+1}) - \BM(\ZB^*)\|^2 \leq S^{t+1} + D^{t}.
	\end{equation}
\end{lemma}
\vspace{-2ex}
\begin{proof}
See Section \ref{app:lemma_G_tp1} in the supplementary material.
\end{proof}

\begin{lemma}\label{lemma_last_lemma}
	Consider the update rule of Algorithm \ref{alg_main} and the definition $\tilde{\WB}=(\bbI+\bbW)/2$. Further, recall the definitions of the sequences $S^{t}$, $T^{t}$, and $D^t$ in \eqref{eqn_def_S}, \eqref{eqn_def_T}, and \eqref{eqn_def_D}, respectively. If each component operator $\BM_{n, i}$ is $\mu$-strongly monotone, $\EBB\|\XB^t - \XB^*\|^2_M$ is upper bounded by
	\begin{shrinkeq}{-1ex}
		\begin{align}
	&\EBB\|\XB^t - \XB^*\|^2_M \leq (2+\frac{4}{\gamma})\EBB\|\ZB^{t+1} - \ZB^t\|^2_{\tilde{\WB}}+ \frac{1}{\mu}T^{t+1}  \nonumber\\
	&\qquad \quad  + 2\EBB\|\QB^{t+1} - \QB^t\|^2 + \frac{4\alpha^2}{\gamma}(S^{t+1} + D^t).
	\end{align}
	\end{shrinkeq}
\end{lemma}
\vspace{-3ex}
\begin{proof}
See Section \ref{app:lemma_last_lemma} in the supplementary material.
\end{proof}
\vspace{-2ex}
Having the above lemmas, we now are ready to state the main theorem. We proceed to show that the Lyapunov function $H^t$ defined as 
\begin{shrinkeq}{-1ex}
	\begin{equation}\label{eqn_def_H}
H^t := \|\XB^t - \XB^*\|^2_{\bbM} + cD^t
\end{equation}
\end{shrinkeq}
converges to zero linearly, where $c$ is a positive constant formally defined in the following theorem.  

\begin{theorem}
	\label{thm_main}
	Consider the proposed DSBA method defined in Algorithm \ref{alg_main}. 
	If each component operator $\BM_{n, i}$ is $1/L$-cocoersive and $\mu$-strongly monotone, by taking the step size $\alpha \leq \frac{1}{24L}$ and $c = \frac{q}{96L^2}$, it holds that
	\begin{shrinkeq}{-2ex}
		\begin{equation}\label{eqn_lin_convg}
	 \EBB[H^{t+1}] \leq \left(1 - \min \left\{\frac{\gamma}{12}, \frac{\mu}{48L}, \frac{1}{3q}, \frac{1}{4}\right\}\right)\EBB[H^t].
	\end{equation}
	\end{shrinkeq}
\end{theorem}
\vspace{-3mm}
\begin{proof}
See Section \ref{app:thm_main} in the supplementary material.
\end{proof}
\vspace{-2ex}
The result in \eqref{eqn_lin_convg} indicates that the Lyapunov function $H^t$ converges to zero Q-linearly in expectation where the coefficient of the linear convergence is a function of graph condition number $\kappa_g\triangleq1/\gamma$, operator condition number $\kappa\triangleq L/\mu$, and number of samples at each node $q$. Indeed, using the definition of $H^t$ in \eqref{eqn_def_H}, the result in Theorem \ref{thm_main} implies R-linear convergence of $\EBB[\|\bbZ^t-\bbZ^*\|^2]$ to zero, i.e., 
\begin{equation*}
\EBB[\|\bbZ^t-\bbZ^*\|_{\tilde{\bbW}}^2] \leq \delta^t (\|\ZB^{0} - \ZB^* \|_{\tilde{\bbW}}^2+\| \QB^{0} - \QB^*\|^2+cD^0),
\end{equation*}
where $\delta := 1 - \min \{\frac{\gamma}{12}, \frac{\mu}{48L}, \frac{1}{3q}, \frac{1}{4}\}$. Note that this result indicates that to obtain an $\epsilon$ accurate solution the number of required iterations is of $\mathcal{O}(\kappa_g+\kappa+q)\log(1/\epsilon)$.

\section{Numerical Experiments}
In this section, we evaluate the empirical performance of DSBA and compare it with several state-of-the-art methods including:
DSA, EXTRA, SSDA, and DLM.
\cite{colin2016gossip} is excluded in comparison since it is sublinear convergent.
Additionally, DSA is implemented using the sparse communication technique developed in Section 5.2.

In all experiments, we set $N=10$ and generate the edges with probability $0.4$.
As to dataset, we use News20-binary, RCV1, and Sector from LIBSVM dataset and randomly split the them into $N$ partitions with equal sizes. 
Further we normalize each data point $\aB_{n, i}$ such that $\|\aB_{n, i}\| = 1$.
We tune the step size of all algorithms and select the ones that give the best performance.
We set $\WB$ to be the Laplacian-based constant edge weight matrix:
$
\bbW = \bbI - \frac{\bbL}{\tau},
$
where $\bbL$ is the Laplacian and $\tau \geq (\lambda_{max}(\bbL))/2$ is a scaling parameter.

\begin{figure}
	\begin{tabular}{c@{}c@{}c}
		\includegraphics[width= 0.33\columnwidth]{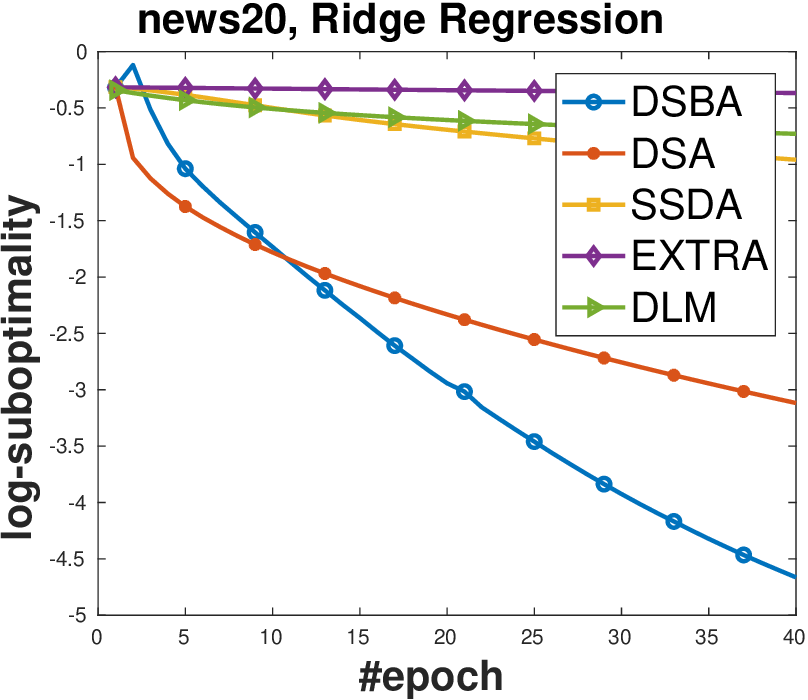}& 	\includegraphics[width= 0.33\columnwidth]{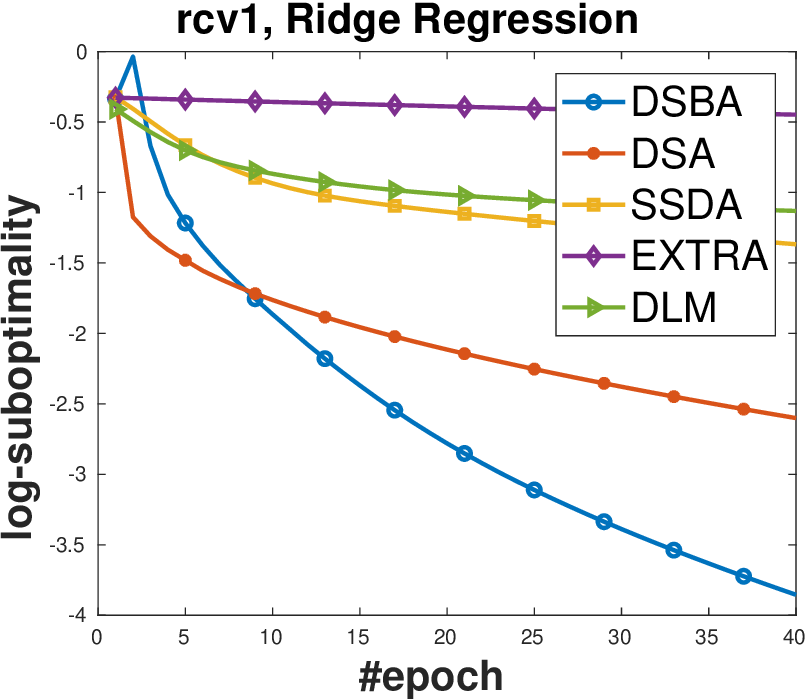} & 	\includegraphics[width= 0.33\columnwidth]{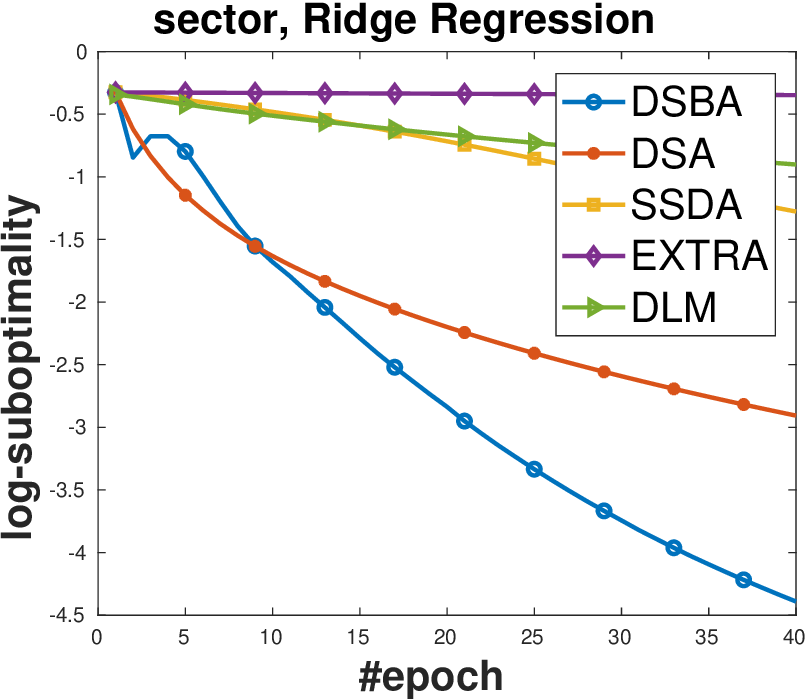}\\
		\includegraphics[width= 0.33\columnwidth]{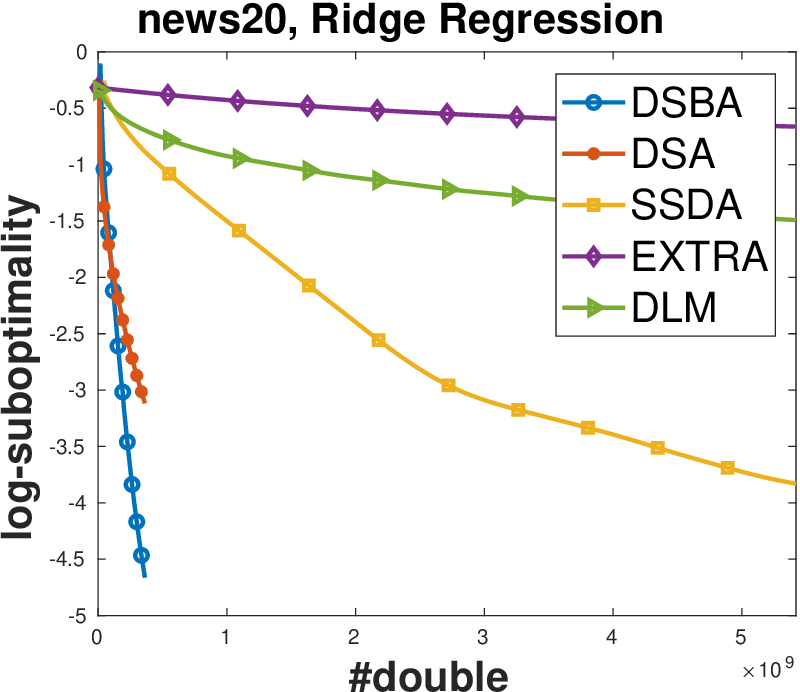}& 	\includegraphics[width= 0.33\columnwidth]{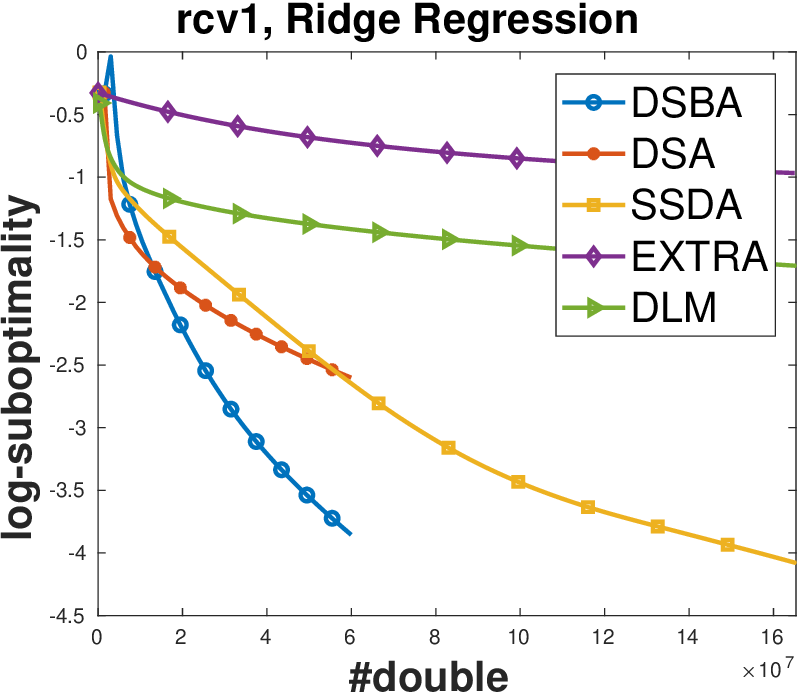} & 	\includegraphics[width= 0.33\columnwidth]{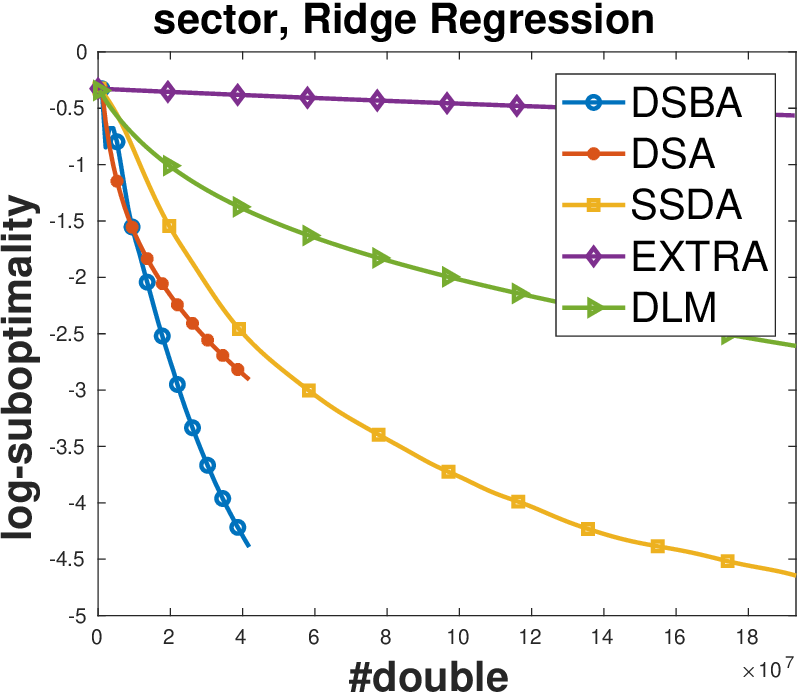}\\
	\end{tabular}
	\vspace{-4mm}
	\caption{Ridge Regression}
	\label{fig_RR}
\end{figure}

\begin{figure}
	\begin{tabular}{c@{}c@{}c}
		\includegraphics[width= 0.33\columnwidth]{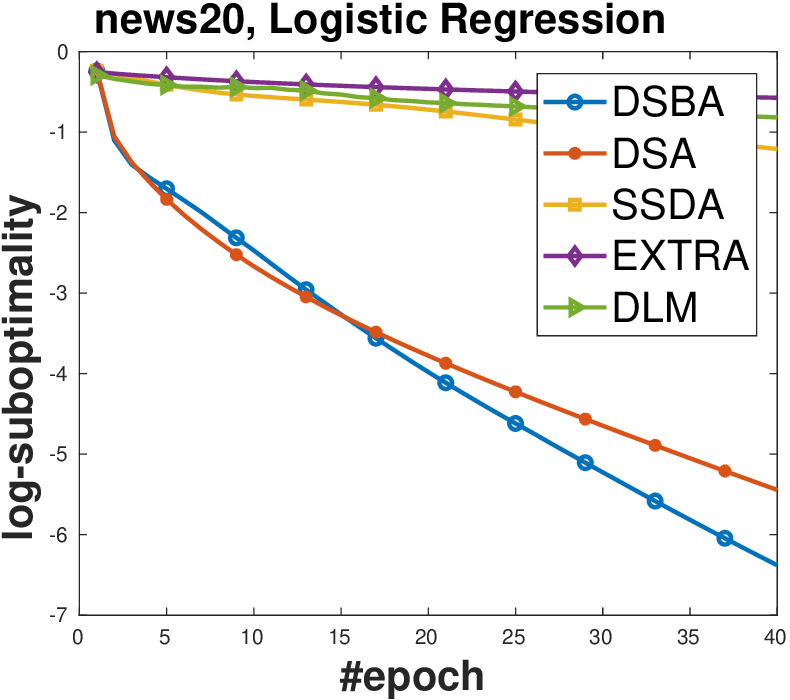}& 	\includegraphics[width= 0.33\columnwidth]{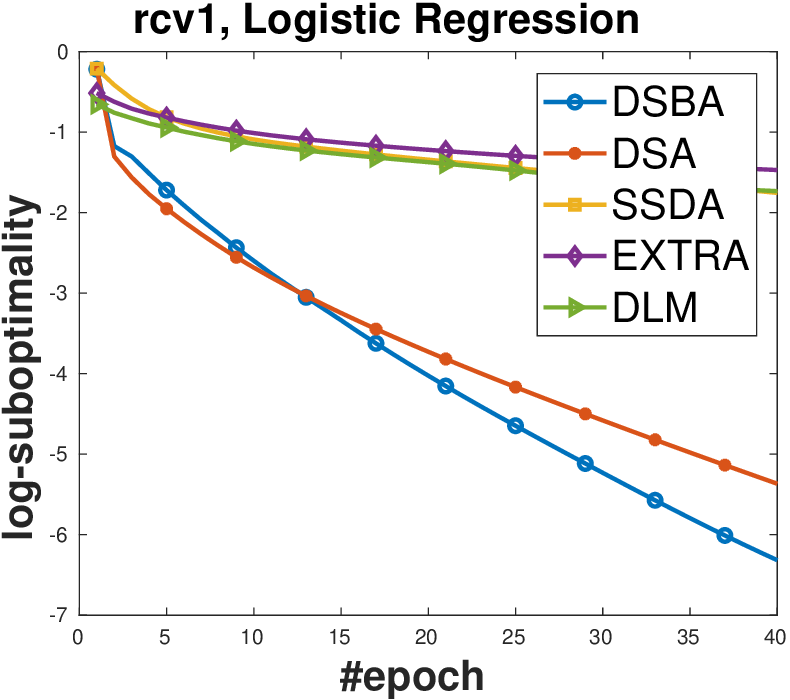} & 	\includegraphics[width= 0.33\columnwidth]{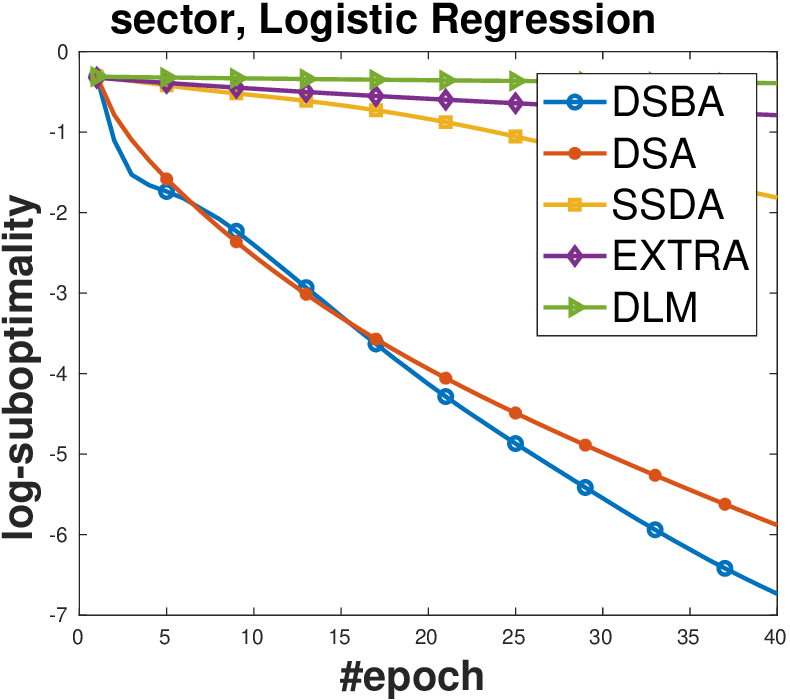}\\
		\includegraphics[width= 0.33\columnwidth]{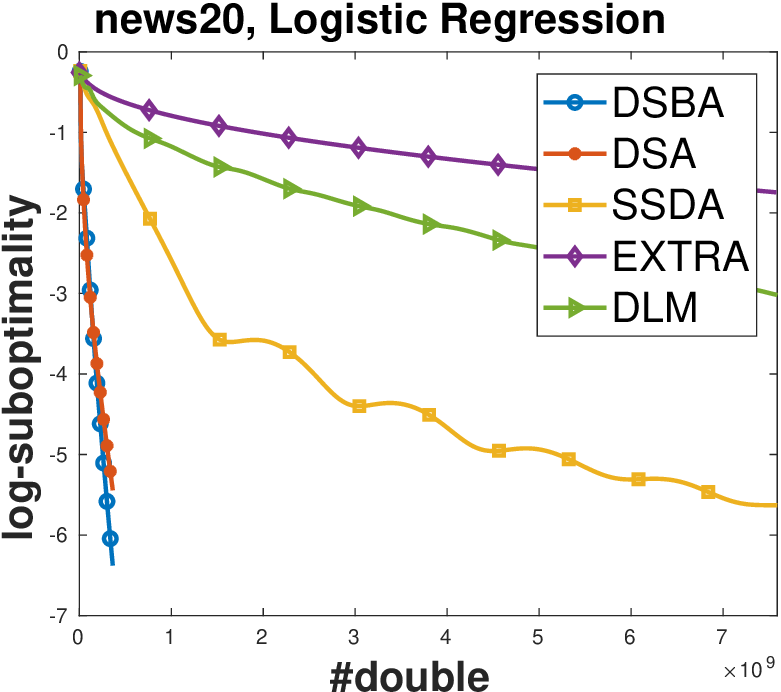}& 	\includegraphics[width= 0.33\columnwidth]{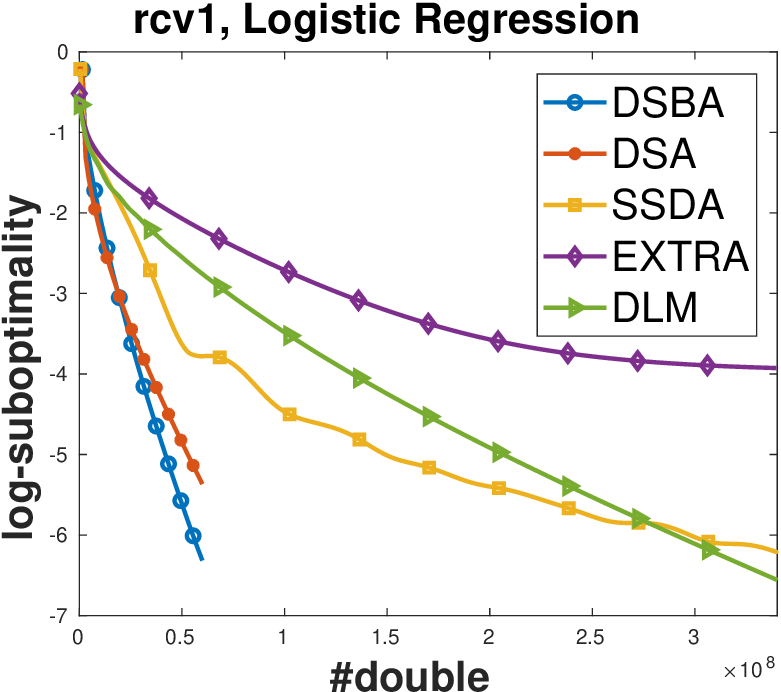} & 	\includegraphics[width= 0.33\columnwidth]{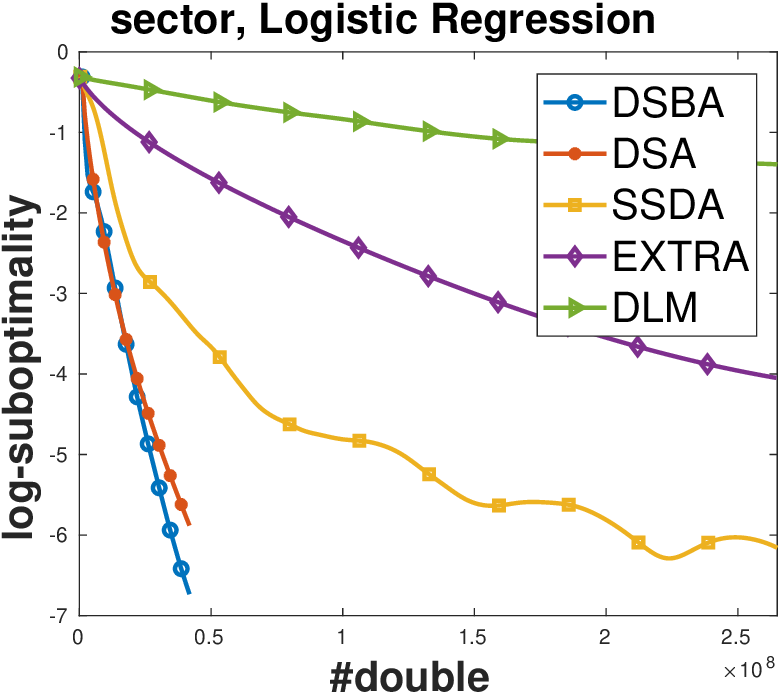}\\
	\end{tabular}
	\vspace{-4mm}
	\caption{Logistic Regression}
	\label{fig_LR}
	\vspace{-4mm}
\end{figure}

We use the \emph{effective pass} over the dataset to measure the cost of computation, which is a common practice in stochastic optimization literature \cite{johnson2013accelerating,defazio2014saga} and is also the one adopted in DSA \cite{mokhtari2016dsa}.
To measure the cost of communication, we let $\CM_n^t$ be the number of DOUBLEs received by node $n$ until iterate $t$ and use $\CM_{max}^t = \max_n \CM_n^t$ as our metric.
Such value $\CM_{max}^t$ captures the communication traffic on the hottest node in the network, which usually is the bottleneck of the learning procedure.

To avoid overfitting and to ensure the strongly monotonicity of an operator $\BM$, we add an $\ell_2$ regularization to all experiments. 
Let $\BM^\lambda = \BM + \lambda \IM$, then the resolvent of $\BM^\lambda$ is closely related to that of $\BM$,
$\JM_{\alpha\BM^\lambda}(\ZB) = \JM_{\rho\alpha\BM}(\rho\ZB)$,
where $\rho = 1 - ({\lambda\alpha})/({1+\lambda\alpha})$ is a scaling factor.
The $\ell_2$-regularization parameter $\lambda$ is set to ${1}/({10 Q})$ in all cases.

\subsection{Ridge Regression}
We define $\BM_{n, i} = (\aB_{n, i}^\top\zB - y_{n, i})\aB_{n, i}$, where $\aB_{n, i}\in \RBB^d$ is the feature vector of a sample in node $n$ and $y_{n, i}\in \RBB$ is its response.
The resolvent of $\alpha\BM_{n, i}$ admits closed form solution: let $z = \frac{\alpha y_{n, i} + \aB_{n, i}^\top\zB}{\alpha + 1}\in\RBB$, then
$\JM_{\alpha\BM_{n, i}}(\zB) = \zB - \alpha(z - y_{n, i})\aB_{n, i}$.
The results are given in Figure \ref{fig_RR}.
It can be seen that the stochastic methods (DSA and DSBA) have the better performance of the deterministic ones.
And DSBA always outperform DSA after several iterations.
\begin{figure}
	\begin{tabular}{c@{}c@{}c}
		\includegraphics[width= 0.33\columnwidth]{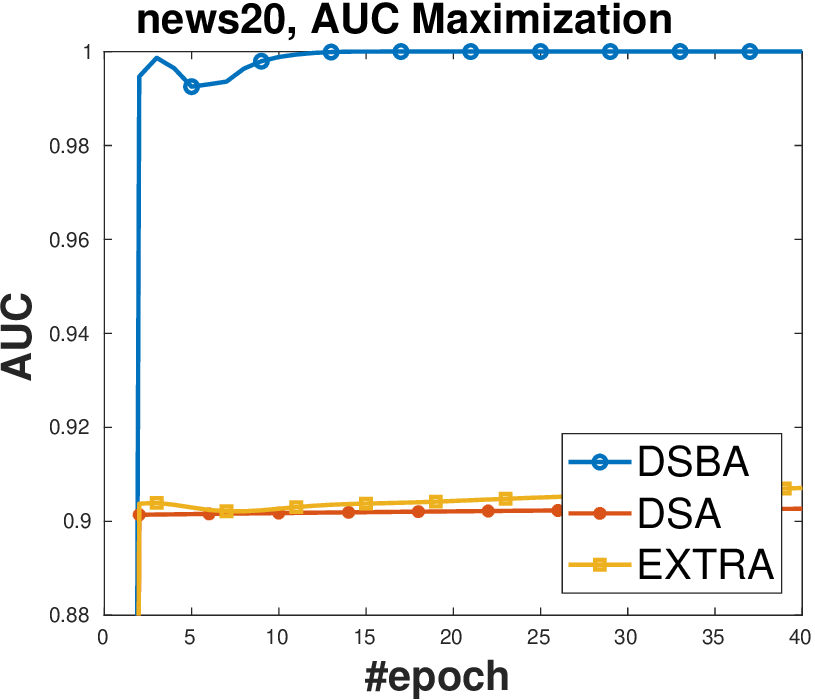}& 	\includegraphics[width= 0.33\columnwidth]{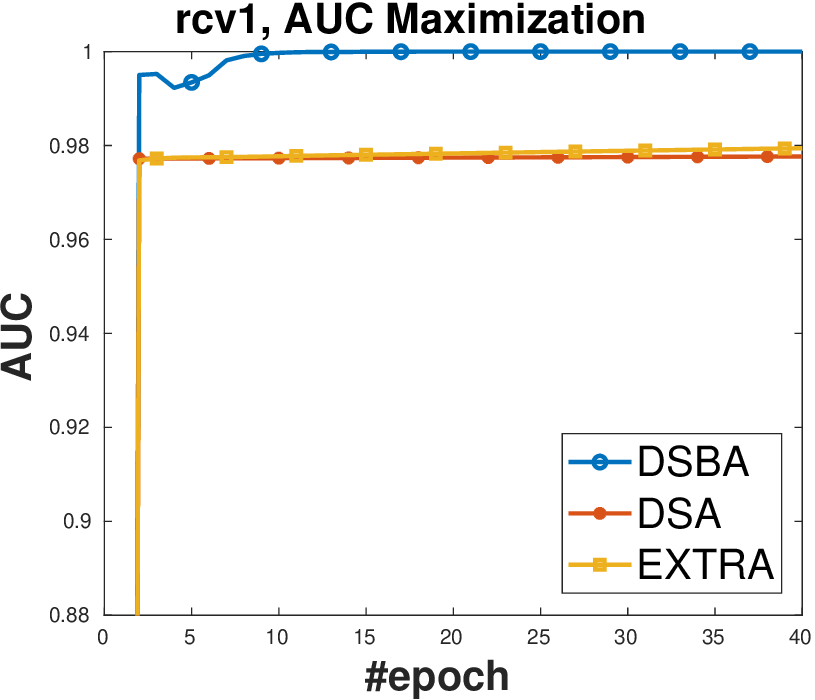} & 	\includegraphics[width= 0.33\columnwidth]{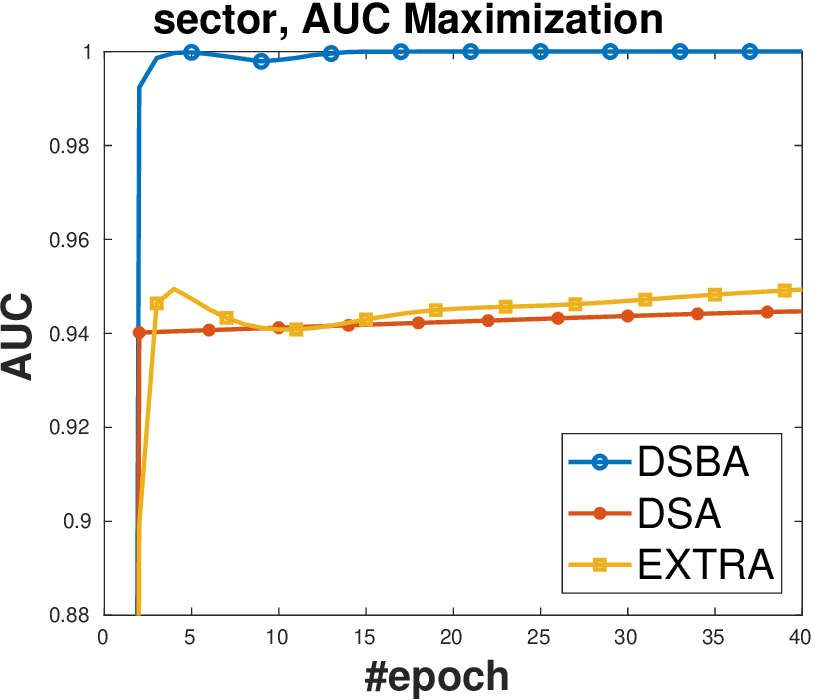}\\
		\includegraphics[width= 0.33\columnwidth]{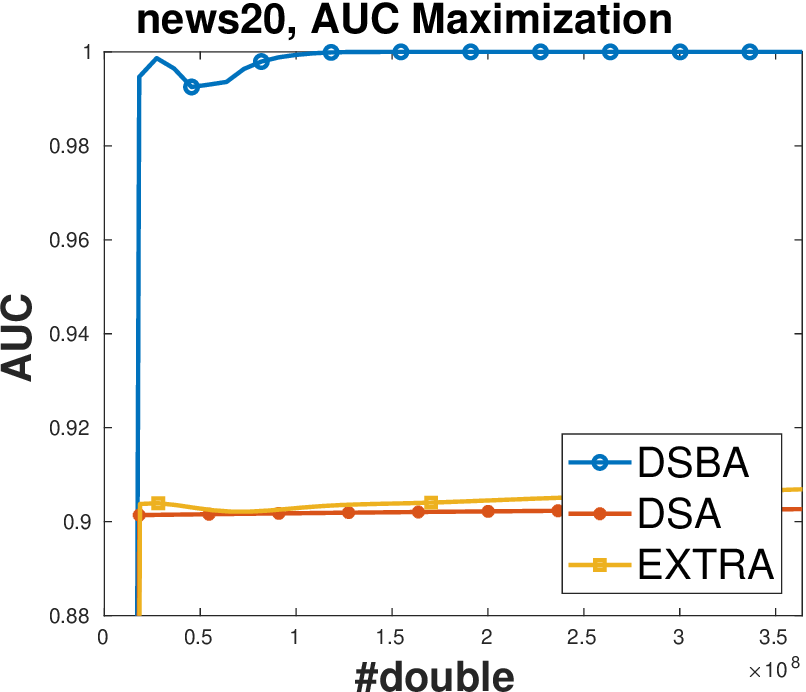}& 	\includegraphics[width= 0.33\columnwidth]{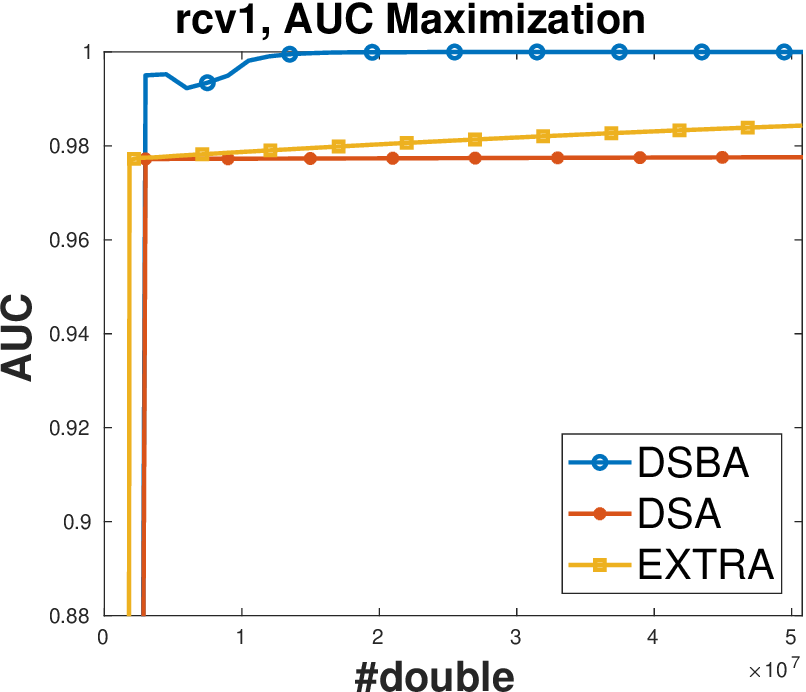} & 	\includegraphics[width= 0.33\columnwidth]{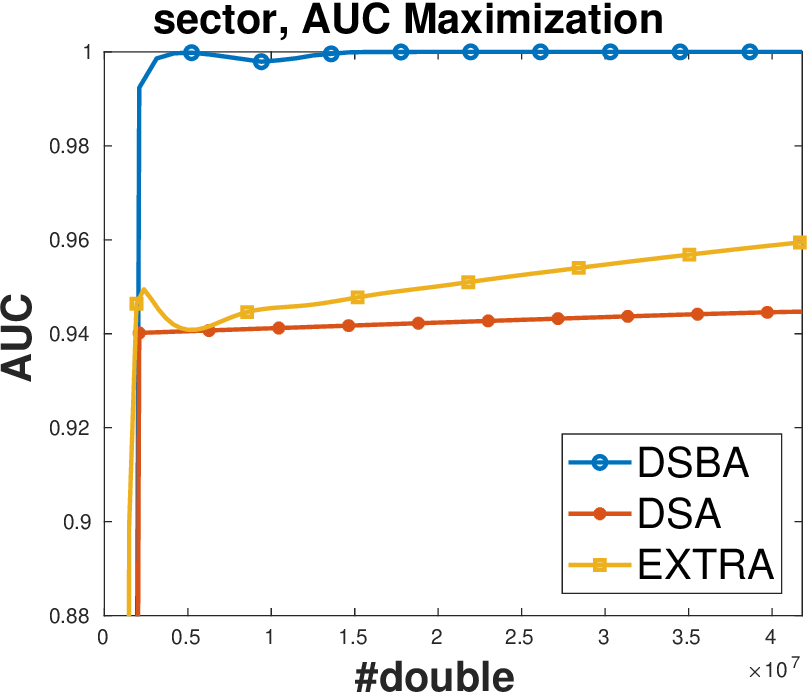}\\
	\end{tabular}
	\vspace{-4mm}
	\caption{$\ell_2$-relaxed AUC maximization}
	\label{fig_AUC}
	\vspace{-5mm}
\end{figure}
\subsection{Logistic Regression}
We define $\BM_{n, i}(\zB) = \frac{-y_{n, i}}{1 + \exp(y_{n, i}\cdot\aB_{n, i}^\top\zB)}\aB_{n, i}$, where $\aB_{n, i}\in \RBB^d$ is the feature vector of a sample and $y_{n, i}\in \{-1, +1\}$ is its class label.
The resolvent $\JM_{\alpha\BM_{n, i}}(\zB)$ does not admit a closed form solution, but can be computed efficiently using a one dimensional newton iteration.
The details are given in the appendix.
We list the experiment results in Figure \ref{fig_LR}.
DSBA has the best performance among all the compared methods, and is able to converge quickly with low communication cost.
\subsection{AUC maximization}
In the $\ell_2$-relaxed AUC maximization, we only compare with DSA and EXTRA because SSDA does not apply and DLM does not converge.
The variable $\zB\in\RBB^{d+3}$ is a $(d+3)$-dimensional augmented vector, where $d$ is the dimension of the dataset.
The component monotone operator $\BM_{n, i}$ is defined in (\ref{eqn_monotone_operator_AUC_p}) and (\ref{eqn_monotone_operator_AUC_n}) in the appendix for positive and negative samples respectively.
Similar to Ridge Regression, the resolvent of $\BM_{n, i}$ also admits a closed form solution, which is explicitly given in the appendix.
The results are given in Figure \ref{fig_AUC}, where DSBA quickly achieves high AUC after a few epochs over the dataset.
\section{Conclusion}
In this paper, we studied the root finding problem of a monotone operator in a decentralized setting.
At a low computation cost, a stochastic algorithm named DSBA is proposed to solve such problem with provably better convergence rate.
By exploiting the dataset sparsity, a sparse communication scheme for implementing DBSA is derived to reduce the communication overhead. 
Our theoretical and numerical results demonstrate the superiority of DSBA over stat-of-the-art deterministic and stochastic decentralized methods. 

\section*{Acknowledgments}
This work is supported by National Natural Science Foundation of China (Grant No: 61472347, 61672376, 61751209), and Zhejiang Provincial Natural Science Foundation of China under Grant No. LZ18F020002.
\bibliography{example_paper}
\bibliographystyle{icml2018}


\newpage

\onecolumn
\section{Supplementary Material}

\subsection{Proof of Lemma \ref{lemma_iteration_recursion}}\label{app:lemma_iteration_recursion}
Note that the update rule (\ref{eqn_update_derivation_t}) can be written as 
	\begin{equation}
	\ZB^{k+1} := \ZB^k + \WB\ZB^k -\tilde{\WB}\ZB^{k-1}- \alpha(\hat{\BM}^{k}(\ZB^{k+1}) - \hat{\BM}^{k-1}(\ZB^{k})),
	\label{eqn_update_derivation_t_proof}
	\end{equation}
	from the definition of $\tilde{\WB}$.
	To prove the first part of the lemma, by summing (\ref{eqn_update_derivation_t_proof}) from $k=1$ to $t$ and (\ref{eqn_update_derivation_0}), one has
	\begin{equation}
		\ZB^{t+1} = (\WB - \tilde{\WB})\sum_{k=0}^{t}\ZB^k + \tilde{\WB}\ZB^t - \alpha\hat{\BM}^{t}(\ZB^{t+1}).
	\end{equation}
	From the definition of $\UB$ and $\QB^{t}$ and the identity $\bbI = 2\tilde{\WB} - \bbW$, we have
	\begin{equation}
	\alpha\hat{\BM}^{t}(\ZB^{t+1}) = \tilde{\WB}(\ZB^t - \ZB^{t+1}) - \bbU\QB^{t+1}.
	\end{equation}
	By subtracting the optimality condition (\ref{eqn_optimality}), we have the result.
	
	From first part, we have
	\begin{align}
	&\langle \ZB^{t+1} - \ZB^*, \alpha [\BM(\ZB^*) - \hat{\BM}^{t}(\ZB^{t+1})]\rangle \nonumber\\
	=& \langle \ZB^{t+1} - \ZB^*, -\tilde{\WB}(\ZB^t - \ZB^{t+1}) + \UB(\QB^{t+1} - \QB^*)\rangle \nonumber\\
	=& \langle \ZB^{t+1} - \ZB^*, \ZB^{t+1} - \ZB^t\rangle_{\tilde{\WB}} + \langle \ZB^{t+1} - \ZB^*, \UB(\QB^{t+1} - \QB^*)\rangle \nonumber\\
	=& \langle \ZB^{t+1} - \ZB^*, \ZB^{t+1} - \ZB^t\rangle_{\tilde{\WB}} + \langle \QB^{t+1} - \QB^t, \QB^{t+1} - \QB^*\rangle,
	\end{align}
	where the last equality uses the definition of $\QB^t$ and that $\UB\ZB^*=\zeroB$.
	By applying the generalized Law of cosines $2\langle a, b\rangle = \|a\|^2 + \|b\|^2 - \|a -b\|^2$ with $a = \XB^{t+1} - \XB^*$ and $b = \XB^{t+1} - \XB^t$, we have the second part.

\subsection{Proof of Lemma \ref{lemma_iteration_recursion_upper_bound}}\label{app:lemma_iteration_recursion_upper_bound}

We have $T^{t+1} \geq \frac{1}{L}S^{t+1}$ from the definition of cocoerciveness.
	Expanding the definition of $\hat{\BM}^{t}(\ZB^{t+1})$, we have
		\begin{align}
		&\EBB\langle\ZB^{t+1} - \ZB^*, \BM(\ZB^*) - \hat{\BM}^{t}(\ZB^{t+1})\rangle \nonumber\\
		= & \sum_{n=1}^{N}  -\EBB_{i_n^t}\langle \zB_{n, i_n^t}^{t+1} - \zB^*, \BM_{n, i_n^t}(\zB_{n, i_n^t}^{t+1}) - \BM_{n, i_n^t}(\zB^*) \rangle \nonumber\\
		& + \EBB_{i_n^t}\langle \zB_{n, i_n^t}^{t+1}- \zB^*, [\BM_{n, i_n^t}(\yB^t_{n, i_n^t})- \BM_{n, i_n^t}(\zB^*)] -[\frac{1}{q}\sum_{i=1}^{q}\BM_{n, i}(\yB^t_{n, i})-\BM_n(\zB^*)]\rangle.
		\end{align}
	The first term is exactly $-\frac{1}{2}T^{t+1}$, and is bounded by $-\frac{1}{2}T^{t+1}\leq -\frac{\theta}{2L}S^{t+1} - \frac{1-\theta}{2} T^{t+1}$ for $0\leq \theta \leq 1$.
	Since 
	\begin{equation}
	\EBB_{i_n^t}\{[\BM_{n, i_n^t}(\yB^t_{n, i_n^t}) - \BM_{n, i_n^t}(\zB^*)] - [\frac{1}{q}\sum_{i=1}^{q}\BM_{n, i}(\yB^t_{n, i})-\BM_n(\zB^*)]\} = \zeroB,
	\end{equation}
	 and $\zB_{n}^{t}$ is independent of $i_n^t$, we have
	 \begin{equation}
	\EBB_{i_n^t}\langle\zB_{n}^{t} - \zB^*, [\BM_{n, i_n^t}(\yB^t_{n, i_n^t}) - \BM_{n, i_n^t}(\zB^*)] - [\frac{1}{q}\sum_{i=1}^{q}\BM_{n, i}(\yB^t_{n, i})-\BM_n(\zB^*)]\rangle = 0.
	\end{equation}
	We bound the second term by
	\begin{align}
	&\sum_{n=1}^{N}\EBB_{i_n^t}\langle\zB_{n, i_n^t}^{t+1}- \zB^*\!, [\BM_{n, i_n^t}(\yB^t_{n, i_n^t}) -\BM_{n, i_n^t}(\zB^*)] -[\frac{1}{q}\sum_{i=1}^{q}\BM_{n, i}(\yB^t_{n, i})-\BM_n(\zB^*)]\rangle \nonumber\\
	= &\!\sum_{n=1}^{N}\EBB_{i_n^t}\langle \zB_{n, i_n^t}^{t+1}- \zB_n^t, [\BM_{n, i_n^t}(\yB^t_{n, i_n^t}) -\BM_{n, i_n^t}(\zB^*)] -[\frac{1}{q}\sum_{i=1}^{q}\BM_{n, i}(\yB^t_{n, i})-\BM_n(\zB^*)]\rangle \nonumber\\
	\leq &\sum_{n=1}^{N} \frac{\eta}{2}\EBB_{i_n^t}\|[\BM_{n, i_n^t}(\yB^t_{n, i_n^t}) - \BM_{n, i_n^t}(\zB^*)] -[\frac{1}{q}\sum_{i=1}^{q}\BM_{n, i}(\yB^t_{n, i})-\BM_n(\zB^*)]\|^2+ \frac{1}{2\eta}\EBB_{i_n^t}\|\zB_{n, i_n^t}^{t+1} - \zB_n^t\|^2 \nonumber\\
	\leq & \sum_{n=1}^{N}\frac{\eta}{2}\EBB_{i_n^t}\|\BM_{n, i_n^t}(\yB^t_{n, i_n^t})- \BM_{n, i_n^t}(\zB^*)\|^2 + \frac{1}{2\eta}\EBB_{i_n^t}\|\zB_{n, i_n^t}^{t+1}- \zB_n^t\|^2\nonumber\\
	= & \frac{1}{2\eta}\EBB\|\ZB^{t+1} - \ZB^t\|^2 + \frac{\eta}{4}D^t,
	\end{align}
	where we use $\langle a, b\rangle \leq \frac{1}{2\eta}\|a\|^2 + \frac{\eta}{2}\|b\|^2$ in first inequality and $\|a - \EBB a\|^2 \leq \|a\|^2$ in the second one.

\subsection{Proof of Lemma \ref{lemma_G_tp1}}\label{app:lemma_G_tp1}
	From the definition of $\hat{\BM}^{t}(\ZB^{t+1})$, on node $n$, we have
	\begin{align}
	\hat{\BM}_n^{t}(\zB_n^{t+1}) &- \BM_n(\zB^*) = [\BM_{n, i_n^t}(\zB_{n, i_n^t}^{t+1}) - \BM_{n, i_n^t}(\zB^*)] - [\BM_{n, i_n^t}(\yB^t_{n, i_n^t}) - \BM_{n, i_n^t}(\zB^*)] + [\frac{1}{q}\sum_{i=1}^{q}\BM_{n, i}(\yB^t_{n, i})-\BM_n(\zB^*)].
	\end{align} 
	Using $\|a+b\|^2 \leq 2\|a\|^2 + 2\|b\|^2$, we have
		\begin{align}
		 & \EBB \|\hat{\BM}^{t}(\ZB^{t+1}) - \BM(\ZB^*)\|^2ã \nonumber€€\\
		 & \leq  \sum_{n=1}^{N} 2\EBB_{i_n^t}\|\BM_{n, i_n^t}(\zB_{n, i_n^t}^{t+1}) - \BM_{n, i_n^t}(\zB^*)\|^2 + 2\EBB_{i_n^t}\|[\BM_{n, i_n^t}(\yB^t_{n, i_n^t}) - \BM_{n, i_n^t}(\zB^*)] - [\frac{1}{q}\sum_{i=1}^{q}\BM_{n, i}(\yB^t_{n, i})-\BM_n(\zB^*)]\|^2 \nonumber \\
		 & \leq  S^{t+1} + D^{t},
	\end{align}
	where the last inequality uses the definition of $D^{t}$ and $S^{t+1}$ and $\|a - \EBB a\|^2 \leq \|a\|^2$.

\subsection{Proof of Lemma \ref{lemma_last_lemma}}\label{app:lemma_last_lemma}

	Expand $\|\XB^t - \XB^*\|^2_M$ by the definition of $\XB^t$ and $\|\cdot\|_M$ and suppose $\ZB^{t+1}$ and $\QB^{t+1}$ are generated from some fixed $i_n^t, n\in[N]$.
	Using $\|a+b\|^2 \leq 2\|a\|^2 + 2\|b\|^2$, we have
	\begin{equation}
	\begin{aligned}
	\|\XB^t - \XB^*\|^2_M &= \|\ZB^t - \ZB^*\|^2_{\tilde{\WB}} + \|\QB^t - \QB^*\|^2 \\
	&\leq 2\|\ZB^{t+1} - \ZB^{t}\|^2_{\tilde{\WB}} + 2\|\ZB^{t+1} - \ZB^*\|^2_{\tilde{\WB}} + 2\|\QB^{t+1} - \QB^t\|^2 + 2\|\QB^{t+1} - \QB^*\|^2.
	\end{aligned}
	\label{eqn_X_upper_bound}
	\end{equation}
	We now bound the second term and last term.
	Using 
	\begin{equation}
	\|\ZB^{t+1} - \ZB^*\|^2_{\tilde{\WB}} \leq \|\ZB^{t+1} - \ZB^*\|^2
	\end{equation}
	since $\tilde{\WB} \preccurlyeq I$, and the $\mu$-strongly monotonicity of $\BM_{n, i_n^t}$, we have 
	\begin{align}
	\|\ZB^{t+1} - \ZB^*\|^2_{\tilde{\WB}} \leq  \frac{1}{\mu}\!\sum_{n=1}^{N}\langle\zB_{n, i_n^t}^{t+1}- \zB^*, \BM_{n, i_n^t}(\zB_{n, i_n^t}^{t+1})- \BM_{n, i_n^t}(\zB^*)\rangle.
	\end{align} 
	
	From the construction of $\QB^{t+1}$ and $\QB^*$, every column of $\QB^{t+1} - \QB^*$ is in $\mathrm{span}(U)$, thus we have
	\begin{equation}
	\gamma\|\QB^{t+1} - \QB^*\|^2 \leq \|U(\QB^{t+1} - \QB^*)\|^2,
	\end{equation}
	where $\gamma$ is the smallest nonzero singular value of $U^2 = \tilde{\WB} - W$.
	From Lemma \ref{lemma_iteration_recursion}, we write
	\begin{align}
	 \|U(\QB^{t+1} - \QB^*)\|^2 
	=\ & \|\alpha[\hat{\BM}^{t}(\ZB^{t+1}) - \BM(\ZB^*)] + \tilde{\WB}(\ZB^{t+1} - \ZB^t)\|^2 \nonumber\\
	\leq\ & 2\alpha^2\|\hat{\BM}^{t}(\ZB^{t+1}) - \BM(\ZB^*)\|^2 + 2\|\ZB^{t+1} - \ZB^t\|_{\tilde{\WB}}^2.
	\end{align}
	Substituting these two upper bounds into (\ref{eqn_X_upper_bound}), we have

	\begin{align}
	\|\XB^t - \XB^*\|^2_M 
	&\leq (2+\frac{4}{\gamma})\|\ZB^{t+1} - \ZB^t\|^2_{\tilde{\WB}} + 2\|\QB^{t+1} - \QB^t\|^2  + \frac{2}{\mu}\sum_{n=1}^{N}\langle\zB_{n, i_n^t}^{t+1} - \zB^*, \BM_{n, i_n^t}(\zB_{n, i_n^t}^{t+1}) - \BM_{n, i_n^t}(\zB^*)\rangle \nonumber\\
	&\qquad+ \frac{4\alpha^2}{\gamma}\|\hat{\BM}^{t}(\ZB^{t+1}) - \BM(\ZB^*)\|^2.
	\end{align}
	Taking expectation and using Lemma \ref{lemma_G_tp1}, we have the result.

\subsection{Proof of Theorem \ref{thm_main}}\label{app:thm_main}

From Lemma \ref{lemma_iteration_recursion} and \ref{lemma_iteration_recursion_upper_bound}, we have
	\begin{align}
	&\EBB\|\XB^{t+1} - \XB^*\|_\bbM^2 - \|\XB^t - \XB^*\|_\bbM^2 + \EBB\|\XB^{t+1} - \XB^t\|_\bbM^2 \nonumber\\
	& = 2\alpha\EBB\langle\ZB^{t+1} - \ZB^*, \BM(\ZB^*) - \hat{\BM}^{t}(\ZB^{t+1})\rangle\nonumber\\
	&\leq \frac{\alpha}{\eta}\EBB\|\ZB^{t+1} - \ZB^t\|^2 + \frac{\eta\alpha}{2}D^t -\frac{\theta\alpha}{L}S^{t+1} - (1-\theta)\alpha T^{t+1}.
	\end{align}
	Also for $D^{t+1}$, we have 
	\begin{align}
	\EBB D^{t+1} =\ & \sum_{n=1}^N\frac{2}{q}\sum_{i=1}^{q}\EBB_{i_n^t}\|\BM_{n, i}(\yB^{t+1}_{n, i}) - \BM_{n, i}(\zB^*)\|^2\nonumber\\
	=\ &\sum_{n=1}^N\frac{2}{q}\sum_{i=1}^{q}\{\frac{1}{q}\|\BM_{n, i}(\zB_{n, i}^{t+1}) - \BM_{n, i}(\zB^*)\|^2+ (1-\frac{1}{q})\|\BM_{n, i}(\yB^{t}_{n, i}) - \BM_{n, i}(\zB^*)\|^2\} \nonumber\\
	=\ & (1 - \frac{1}{q})D^t + \frac{1}{q}S^{t+1}.
	\end{align}
	By adding $cD^{t+1}$ and rearranging terms, we have
	\begin{align}
\EBB[\|\XB^{t+1} - \XB^*\|^2_{M} + cD^{t+1}] 
	\leq & \|\XB^t -\XB^*\|_\bbM^2 - \EBB\|\XB^{t+1} - \XB^t\|_\bbM^2 + (1 - \frac{1}{q})cD^t + \frac{c}{q}S^{t+1} \nonumber\\
	& + \frac{\alpha}{\eta}\EBB\|\ZB^{t+1} - \ZB^t\|^2 + \frac{\eta\alpha}{2}D^t -\frac{\theta\alpha}{L}S^{t+1} - (1-\theta)\alpha T^{t+1}.
	\end{align}
	If we further have 
	\begin{align}
	(1-\delta)[\|\XB^t - \XB^*\|^2_{M} + cD^t] 
	\geq & \|\XB^t - \XB^*\|_\bbM^2 -\EBB\|\XB^{t+1}-\XB^t\|_\bbM^2 + (1 - \frac{1}{q})cD^t + \frac{c}{q}S^{t+1} \nonumber\\
	& + \frac{\alpha}{\eta}\EBB\|\ZB^{t+1} - \ZB^t\|^2 + \frac{\eta\alpha}{2}D^t -\frac{\theta\alpha}{L}S^{t+1} - (1-\theta)\alpha T^{t+1},
	\end{align}
	then we have the result.
	The above inequality is equivalent to 
	\begin{equation}
	\begin{aligned}
	&(\frac{c}{q}-c\delta - \frac{\alpha\eta}{2})D^t + (\frac{\alpha\theta}{L} - \frac{c}{q})S^{t+1} + \alpha(1-\theta)T^{t+1} \\
	\geq & \underbrace{\delta\|\XB^t - \XB^*\|_\bbM^2 - \|\XB^{t+1} - \XB^t\|_\bbM^2 + \frac{\alpha}{\eta}\EBB\|\ZB^{t+1} - \ZB^t\|^2}_\Lambda,
	\end{aligned}
	\end{equation}
	and hence a sufficient condition is that an upper bound of the right hand side is less than the left hand side.
	
	To bound $\Lambda$, using Lemma \ref{lemma_last_lemma} for the first term, the definition of $\|\XB^{t+1} - \XB^t\|_\bbM^2$ for the second term, and 
	\begin{equation}
	\frac{1}{2} \|\ZB^{t+1} - \ZB^t\|^2 \leq \|\ZB^{t+1} - \ZB^t\|_{\tilde{\WB}}^2
	\end{equation}
	for the third term since $\frac{1}{2}I\preccurlyeq \tilde{\WB}$, we have 
	\begin{align}
	\Lambda \leq& \delta[(2+\frac{4}{\gamma})\EBB\|\ZB^{t+1} - \ZB^t\|_{\tilde{\WB}}^2 + \frac{1}{\mu}T^{t+1} + 2\EBB\|\QB^{t+1} - \QB^t\|^2 + \frac{4\alpha^2}{\gamma}(S^{t+1} + D^t)] \nonumber\\
	& - \EBB\|\ZB^{t+1} - \ZB^t\|^2_{\tilde{\WB}} - \EBB\|\QB^{t+1} - \QB^{t}\|^2  + \frac{2\alpha}{\eta}\EBB\|\ZB^{t+1} - \ZB^t\|_{\tilde{\WB}}^2 .
	\end{align}
	Uniting like terms gives us the following sufficient condition for Theorem \ref{thm_main} to stand:
	\begin{align}
	&(\frac{c}{q}-c\delta - \frac{\alpha\eta}{2} - \frac{4\delta\alpha^2}{\gamma})D^t + (\frac{\alpha\theta}{L} - \frac{c}{q} - \frac{4\delta\alpha^2}{\gamma})S^{t+1} + (\alpha(1-\theta) - \frac{\delta}{\mu})T^{t+1} \nonumber\\
	&+ (1 - 2\delta)\EBB\|\QB^{t+1} - \QB^{t}\|^2 + (1 - (2+\frac{4}{\gamma})\delta - \frac{2\alpha}{\eta})\EBB\|\ZB^{t+1} - \ZB^t\|^2_{\tilde{\WB}}\geq 0.
	\end{align}
	Since every term in the above inequality is nonnegative, this inequality holds when every bracket is nonnegative.
	Let
	\begin{equation}
	\alpha = \frac{\tau}{L}, \eta = 4\alpha, \theta = \frac{1}{2}, c = \frac{mq}{L^2},
	\end{equation}
	where $\tau$ and $m$ are constant to be set.
	The non-negativity of of the first two brackets equivalents to 
	\begin{equation}
	\begin{cases}
	c(\frac{1}{3q} - \delta) + \frac{2m}{3L^2} - \frac{2\tau^2}{L^2} - \frac{\delta}{\gamma}\frac{4\tau^2}{L^2} \geq 0 \\
	\frac{\tau}{2L^2} - \frac{m}{L^2} - \frac{\delta}{\gamma}\frac{4\tau^2}{L^2} \geq 0
	\end{cases}
	\end{equation}
	Taking $\tau = \frac{1}{24}, m = \frac{1}{96}, \delta \leq \min\{\frac{\gamma}{12}, \frac{\mu}{48L}, \frac{1}{3q}, \frac{1}{4}\}$, we have the result.
\subsection{Resolvent of Logistic Regression}
In Logistic Regression, each component operator $\BM_{n, i}$ is defined as $\BM_{n, i}(\zB) = \frac{-y_{n, i}}{1 + \exp(y_{n, i}\cdot\aB_{n, i}^\top\zB)}\aB_{n, i}$, where $\aB_{n, i}\in \RBB^d$ is the feature vector of a sample and $y_{n, i}\in \{-1, +1\}$ is its class label.
The resolvent, $\JM_{\alpha\BM_{n, i}}(\zB)$, does not admit a closed form solution, but can be computed efficiently by the following newton iteration: let $a_0 = 0$, $b = \aB_{n, i}^\top\zB$
\begin{equation}
e_k = \frac{-y_{n, i}}{1+\exp(y_{n, i} a_k)} ~~\mathrm{and}~~ a_{k+1} = a_k - \frac{\alpha e_k + a_k - b}{1 - \alpha y_{n, i} e_k - \alpha e_k^2}.
\end{equation}
When the iterate converges, the resolvent is obtain by
\begin{equation}
	\JM_{\alpha\BM_{n, i}}(\zB) = \zB - (b - a_{k})\aB_{n, i}.
\end{equation}
In our experiments, $20$ newton iteration is sufficient for DSBA.
\subsection{Resolvent of AUC maximization}
In the $\ell_2$-relaxed AUC maximization, the variable $\zB\in\RBB^{d+3}$ is a $d+3$-dimensional augmented vector, where $d$ is the dimension of the dataset.
For simplicity, we decompose $\zB$ as $\zB = [\wB^\top; a; b; \theta]$ with $\wB \in \RBB^d$, $a \in \RBB$, $b \in \RBB$, $\theta \in \RBB$.
For a positive sample, i.e. $y_{n, i} = +1$, the component operator $\BM_{n, i}$ is then defined as
\begin{equation}
\label{eqn_monotone_operator_AUC_p}
\BM_{n, i}(\zB) =
\begin{bmatrix}
2(1-p)((\aB_{n, i}^\top\wB - a) - (1+\theta))\aB_{n, i} \\
-2(1-p)(\aB_{n, i}^\top\wB - a) \\
0 \\
2p(1-p)\theta+2(1-p)\aB_{n, i}^\top\wB
\end{bmatrix}
\end{equation}
and for a negative sample, i.e. $y_{n, i} = -1$
\begin{equation}
\label{eqn_monotone_operator_AUC_n}
\BM_{n, i}(\zB) =
\begin{bmatrix}
2p((\aB_{n, i}^\top\wB - b) + (1+\theta))\aB_{n, i} \\
0 \\
-2p(\aB_{n, i}^\top\wB - b) \\
2p(1-p)\theta-2p\aB_{n, i}^\top\wB
\end{bmatrix}
\end{equation}
where $p = \frac{\#\mathrm{positive~samples}}{\#\mathrm{samples}}$ is the positive ratio of the dataset.
Similar to RR, the resolvent of $\BM_{n, i}$ also admits a closed form solution, which we now derive.
For a positive sample, define
\begin{equation}
\bbA^+ = \begin{bmatrix}
1+2(1-p)\alpha &-2(1-p)\alpha&0&-2(1-p)\alpha \\
-2(1-p)\alpha & 1+2(1-p)\alpha&0&0 \\
0&0&1&0 \\
2(1-p)\alpha&0&0&1+2p(1-p)\alpha
\end{bmatrix}
\end{equation}
and 
\begin{equation}
\bbb^+ = \begin{bmatrix}
\aB_{n, i}^\top\wB + 2(1-p)\alpha\\
a\\
b\\
\theta
\end{bmatrix}.
\end{equation}
Let $\bbb_r^+ = (\bbA^+)^{-1}\bbb^+ \in \RBB^4$ and decompose it as $\bbb_r^+ = [z_r^+; a_r^+; b_r^+; \theta^+_r]$.
The resolvent is obtain as
\begin{equation}
\label{eqn_resolvent_AUC_p}
\JM_{\alpha\BM_{n, i}}(\zB) =\bbz^+_r = \begin{bmatrix}
[\wB - 2(1-p)\alpha[(z_r^+ - a) - (1+\theta)]\aB_{n, i}\\
a_r^+\\
b_r^+\\
\theta_r^+
\end{bmatrix}
\end{equation}
We can do the similar derivation for a negative sample. Define
\begin{equation}
\bbA^- = \begin{bmatrix}
1+2p\alpha&0&-2p\alpha&2p\alpha \\
0&1&0&0\\
-2p\alpha&0&1+2p\alpha&0 \\
-2p\alpha&0&0&1+2p(1-p)\alpha
\end{bmatrix}
\end{equation}
and 
\begin{equation}
\bbb^+ = \begin{bmatrix}
\aB_{n, i}^\top\wB - 2p\alpha\\
a\\
b\\
\theta
\end{bmatrix}
\end{equation}
Let $\bbb_r^- = (\bbA^-)^{-1}\bbb^- \in \RBB^4$ and decompose it as $\bbb_r^- = [z_r^-; a_r^-; b_r^-; \theta^-_r]$
The resolvent is obtain as
\begin{equation}
\label{eqn_resolvent_AUC_n}
\JM_{\alpha\BM_{n, i}}(\zB) =\bbz^+_r = \begin{bmatrix}
[\wB - 2p\alpha[(z_r^- - b) - (1+\theta)]\aB_{n, i}\\
a_r^-\\
b_r^-\\
\theta_r^-
\end{bmatrix}.
\end{equation}

\end{document}